%% file: theory_paper_scorer_agnostic_revision.tex
\documentclass[10pt,twocolumn]{article}

\usepackage[margin=0.75in]{geometry}
\usepackage{amsmath,amssymb,amsthm,mathtools}
\usepackage{bm}
\usepackage{enumitem}
\usepackage{microtype}
\usepackage{booktabs}
\usepackage{multirow}
\usepackage{graphicx}
\usepackage{xcolor}
\usepackage[hidelinks]{hyperref}
\usepackage{algorithm}
\usepackage{algpseudocode}
\usepackage{tabularx}
\usepackage{caption}
\usepackage[numbers,sort&compress]{natbib}

\newtheorem{theorem}{Theorem}
\newtheorem{lemma}[theorem]{Lemma}
\newtheorem{proposition}[theorem]{Proposition}
\newtheorem{corollary}[theorem]{Corollary}
\theoremstyle{definition}
\newtheorem{assumption}{Assumption}

\theoremstyle{remark}
\newtheorem{remark}[theorem]{Remark}

\newcommand{\R}{\mathbb{R}}
\newcommand{\E}{\mathbb{E}}
\newcommand{\Prob}{\mathbb{P}}
\newcommand{\cX}{\mathcal{X}}
\newcommand{\KRR}{\mathrm{KRR}}
\newcommand{\FAR}{\mathrm{FAR}}
\newcommand{\pcap}{p_{\mathrm{cap}}}

\title{\textbf{Beyond the Simplex: Balanced Prototype Geometry for Scorer-Agnostic Open-Set Recognition}}

\author{
\parbox[t]{0.42\textwidth}{\centering
Mayank Sharma \\
\textit{Indian Institute of Technology Jodhpur} \\
\texttt{b23es1023@iitj.ac.in}
}
\hfill
\parbox[t]{0.42\textwidth}{\centering
Rohit Kumar Mourya \\
\textit{Indian Institute of Technology Jodhpur} \\
\texttt{b23es1029@iitj.ac.in}
}
}

\date{}

\begin{document}
\maketitle

%% ============================================================
%% ABSTRACT
%% ============================================================
\begin{abstract}
Open-set recognition (OSR) requires a classifier to reject inputs from unseen
classes which is essential in safety-critical settings such as medical imaging.
Simplex based methods, which fix class prototypes at the vertices of a regular
simplex and then reject via a distance-ratio score, perform well empirically but lack
theoretical justification, and existing analysis applies only when the embedding
dimension \(d\) is at least \(C-1\), which is the regime in which a regular simplex
exists.

We give a theoretical account of simplex-ratio OSR that holds in every embedding
dimension, including \(d < C-1\). Our analysis centers on balanced equal-norm codes: prototype configurations
with equal lengths and zero sum, which exist for all \(d \ge 2\) and include
the regular simplex as a special case. For these codes
we show that an auxiliary squared ratio score has sublevel sets that are exact
unions of Euclidean balls, which in turn bracket the acceptance region of the
operational score; and we prove a sharp dichotomy: the prototypes attain
one-distance symmetry, behaving like a regular simplex, if and only if
\(d \ge C-1\), with controlled degradation governed by an explicit defect
parameter below that threshold. We further show the false-acceptance rate decays
exponentially in \(d\) under natural isotropy assumptions, and that the
operational score is globally Lipschitz with compact acceptance regions.

Empirically, we study balanced prototype geometry as both an analytic tool and
a representation-learning prior, rather than as a stand-alone state-of-the-art
detector. Across CIFAR and MedMNIST open-set splits, the geometry provides
useful structure, but OSR performance remains strongly dependent on the scoring
rule: raw ratio scores typically underperform nearest-neighbor and logit-based
alternatives.
\end{abstract}


%% ============================================================
%% 1. INTRODUCTION
%% ============================================================
\section{Introduction}\label{sec:intro}

Open-set recognition (OSR) formalizes the practical scenario where a classifier must simultaneously identify samples from known training classes and reject inputs from previously unseen classes at test time \cite{scheirer2012toward,geng2021recent}.
Unlike closed-set classification, which assumes the label space is shared between training and testing, OSR acknowledges the open-world nature of real deployments and requires explicit mechanisms for novelty detection \cite{sun2023comprehensive}.

This capability is especially critical in medical imaging, where rare diseases, imaging artifacts, and previously unidentified pathologies can arise at any time \cite{wang2023uncertainty}.
A model that confidently assigns a known-class label to a truly novel input can lead to misdiagnosis with severe clinical consequences.
Consequently, there is growing interest in OSR methods that are both accurate and theoretically well-founded for safety-critical applications.

\textbf{Neural collapse and simplex classifiers.}
Recent theoretical insights reveal that deep neural classifiers, during the terminal phase of training, exhibit a striking geometric convergence known as \emph{neural collapse} (NC) \cite{papyan2020prevalence}: last-layer features collapse to their class means, and these means align with the vertices of a simplex equiangular tight frame (ETF).
This observation has motivated a family of classifiers that fix the class prototypes at simplex ETF vertices from the outset, simultaneously maximizing both Euclidean and angular inter-class margins \cite{cevikalp2024dsc}.

The Deep Simplex Classifier (DSC) \cite{cevikalp2024dsc} trains features to cluster around fixed simplex vertices using an intra-class loss, with a triplet-based outlier-exposure term for background rejection.
The Uncertainty-aware Deep Simplex Classifier (UCDSC) \cite{aditya2025ucdsc} extends DSC by adding an uncertainty-aware regularization term $L_u$, inspired by the distance-ratio scoring of Cao et al.~\cite{cao2021openset}, that penalizes the open space between class prototypes.
UCDSC reports competitive performance on multiple MedMNIST benchmarks \cite{yang2021medmnist} and the Augmented Skin Conditions dataset \cite{naqvi2023skin}.

\textbf{The theory gap.}
Despite these empirical results, neither DSC nor UCDSC provides formal theoretical analysis.
Fundamental questions remain unanswered:
\emph{Why does the simplex-ratio score $U$ work for open-set detection?
Under what conditions is the acceptance region compact?
How does the false acceptance rate (FAR) scale with embedding dimension?
Why does the $L_u$ term help even without background data?
How should the threshold $\theta$ be chosen to guarantee a target known-class recall?}
Moreover, existing simplex-based methods assume $d\ge C-1$ in order to realize a regular simplex, yet practical embedding dimensions can be small relative to the number of classes.
\emph{Does the rejection geometry collapse when $d<C-1$, or do the core guarantees survive?}

\textbf{Our contributions.}
We develop a unified theoretical framework for simplex-ratio OSR that works for \emph{all} prototype configurations, not only the regular simplex, and that covers \emph{all} embedding dimensions including $d<C-1$.
Our main contributions are:

\begin{enumerate}[leftmargin=1.6em,itemsep=2pt]
\item \textbf{General Lipschitz analysis and ball control (Section~\ref{sec:concentration}).}
For \emph{arbitrary distinct} prototypes (no simplex assumption), we establish global Lipschitz continuity of $U$ with constant $4/\Delta$ where $\Delta$ is the minimum inter-prototype distance, and prove that $U$-acceptance regions are compact sets controlled by explicit prototype-centered inner and outer ball unions (Theorem~\ref{thm:generic-ball}).
For the regular-simplex special case this specializes to $4/D$.
Under assumption A5, we derive class-conditional quantile bounds enabling population recall control (Theorem~\ref{thm:far-krr}(i)).

\item \textbf{Exact $U_2$ geometry for balanced equal-norm codes (Section~\ref{sec:geometry}).}
We introduce \emph{balanced equal-norm codes} ($\|s_j\|=R$, $\sum_j s_j=0$), which exist in every dimension $d\ge 2$ and include regular simplices as a special case.
For these codes, we characterize $U_2$-sublevel sets as exact unions of at most $C$ closed Euclidean balls with explicit centers and radii, and show that $U$-acceptance regions lie sandwiched between corresponding ball unions (Theorem~\ref{thm:geometry-U2}).

\item \textbf{Sharp dichotomy and defect parameter (Section~\ref{sec:dichotomy}).}
We introduce the simplex-defect parameter $\lambda_j=\sqrt{B_j}/A_j\ge 1$, where $A_j$ and $B_j$ are the mean and mean-squared rival distances from class $j$.
We prove a sharp dichotomy: $\lambda_j=1$ for all $j$ if and only if the prototypes form a regular simplex, which requires $d\ge C-1$ (Theorem~\ref{thm:dichotomy}).
Below this threshold, we establish a local equivalence $U\asymp\sqrt{U_2}$ with explicit constants depending on $\lambda_j$ (Theorem~\ref{thm:local-comparison}).

\item \textbf{Concrete low-dimensional construction (Section~\ref{sec:polygon}).}
We provide a universal construction---the regular $C$-gon embedded in $\R^d$ for any $d\ge 2$---with fully explicit formulas for all geometric statistics, showing that the framework has substantive content in every low-dimensional regime.

\item \textbf{Conditional FAR bounds (Section~\ref{sec:far}).}
Under normalized unknown directions, we prove a FAR upper bound of the form $C\exp(-c(\theta)d)+C\,\delta_{\mathrm{cap}}$ and, under exact isotropy, a conservative spherical-cap lower bound.
Under the radial-isotropic unknown-feature model, we prove an exponential-in-$d$ FAR upper bound (Theorem~\ref{thm:far-krr}(iii)--(iv)).

\item \textbf{UCDSC loss analysis (Section~\ref{sec:loss-analysis}).}
We show that the three UCDSC loss components play complementary but different roles: Proposition~\ref{prop:intra} controls the average training score, Theorem~\ref{thm:combined} gives a conditional empirical separation guarantee for training knowns versus auxiliary background samples, and Proposition~\ref{prop:grad} together with Corollary~\ref{cor:lu} provide a local negative-gradient interpretation for $L_u$.

\item \textbf{Practical corollaries (Section~\ref{sec:corollaries}).}
We derive closed-form guidance for threshold selection, expand-factor effects, and a sufficient embedding dimension for a target FAR.
We also provide a concrete, validation-based proxy for estimating the rescaled score-tail parameter $\sigma$ from data (Algorithm~\ref{alg:sigma}), making the population threshold rule operational up to the usual finite-sample caveats.

\item \textbf{Scorer-agnostic empirical validation (Section~\ref{sec:experiments}).}
We implement a validation-calibrated representation-by-scorer protocol. Each representation is trained once, then evaluated with \(U\)-unsquared, \(U_2\)-squared, min-distance, KNN50, MSP, MaxLogit, Energy, ODIN, ViM, ReAct, and OpenMax where applicable. The completed slices support presenting balanced prototype geometry as an analytic and representation-learning prior rather than as a claim that the raw ratio score dominates standard post-hoc detectors.
\end{enumerate}

\paragraph{Summary of theoretical results.}
The analysis establishes five facts used throughout the paper.
First, for any distinct prototypes, the operational score $U(z)=d_{\min}(z)/\mu(z)$ is globally Lipschitz with constant $4/\Delta$, and its acceptance regions are compact sets bounded by explicit prototype-centred ball unions (Theorem~\ref{thm:generic-ball}).
Second, for balanced equal-norm codes, the auxiliary score $U_2$ has exact sublevel sets given by a union of at most $C$ Euclidean balls, with matching inner and outer ball controls for $U$ (Theorem~\ref{thm:geometry-U2}).
Third, within the balanced equal-norm class, the condition $\lambda_j=1$ for all $j$ is equivalent to regular-simplex geometry and therefore requires $d\ge C-1$ (Theorem~\ref{thm:dichotomy}).
Fourth, under the normalized-route assumptions A6/A6$'$, the false-acceptance probability admits a dimension-dependent upper bound involving spherical caps and the cap-discrepancy term $\delta_{\mathrm{cap}}$ (Theorem~\ref{thm:far-krr}).
Finally, the loss analysis gives deterministic training/background score separation under explicit closeness conditions and a local negative-gradient interpretation of the $L_u$ term (Proposition~\ref{prop:intra}, Theorem~\ref{thm:combined}, Proposition~\ref{prop:grad}, and Corollary~\ref{cor:lu}).

\textbf{Paper structure.}
Section~\ref{sec:related} reviews related work.
Section~\ref{sec:prelim} states the formal setup, introducing the hierarchy of prototype assumptions (A1/A1$'$/A1$''$) and connecting the two score variants ($U$ and $U_2$).
Section~\ref{sec:concentration} establishes Lipschitz analysis and ball control for general prototypes.
Section~\ref{sec:geometry} develops exact $U_2$ geometry under balanced equal-norm codes, the sharp dichotomy theorem, local comparison bounds, and the $C$-gon construction.
Sections~\ref{sec:far}--\ref{sec:corollaries} develop FAR bounds, loss analysis, and practical corollaries.
Section~\ref{sec:experiments} reports scorer-agnostic validation experiments and failure analysis for the raw ratio scores.
Section~\ref{sec:conclusion} concludes.
Full proofs appear in the Appendix.

%% ============================================================
%% 2. RELATED WORK
%% ============================================================
\section{Related Work}\label{sec:related}

\textbf{Open-set recognition.}
The foundations of OSR were laid by Scheirer et al.~\cite{scheirer2012toward}, who formalized open-space risk and later proposed compact abating probability (CAP) models \cite{scheirer2014probability}.
Subsequent work explored diverse rejection mechanisms: OpenMax \cite{bendale2016towards} adapts extreme-value calibration to softmax outputs, DOC \cite{shu2017doc} uses per-class sigmoid thresholds, and reconstruction-based approaches like CROSR \cite{yoshihashi2019classification} and C2AE \cite{oza2019c2ae} leverage reconstruction error as a novelty signal.
Distance-based methods represent another strand, including NNO \cite{bendale2015towards}, OSNN \cite{mendes2017nearest}, EVM \cite{rudd2018extreme}, and class anchor clustering \cite{miller2021class}.

\textbf{Prototype and simplex methods.}
More recent work has turned toward prototype-based approaches for their interpretability and geometric structure. GCPL \cite{yang2018robust} learns class prototypes with an explicit loss, while RPL \cite{chen2020learning} introduces reciprocal points to explicitly model extra-class space. This idea was refined by ARPL \cite{chen2022adversarial}, which adds adversarial margin constraints. In parallel, methods like DIAS \cite{moon2022difficulty} and OMCL \cite{liu2023learning} use feature generation and angular margins to control open space. More directly relevant to this work, DSC \cite{cevikalp2024dsc} fixes prototypes at simplex ETF vertices and employs intra-class clustering with triplet-based outlier exposure \cite{hendrycks2019deep}. UCDSC \cite{aditya2025ucdsc} extends this by adding an uncertainty-aware regularization that penalizes high inter-prototype distances. Most recently, Vishal et al.~\cite{vishal2026dmdsc} extend this line of work with class-specific dynamic margins that scale with label frequency, addressing extreme class imbalance in medical OSR; our theoretical framework targets the balanced equal-norm geometry itself and is orthogonal to such margin scheduling.

\textbf{Neural collapse and simplex geometry.}
The simplex structure underlying DSC and UCDSC connects to the broader phenomenon of neural collapse \cite{papyan2020prevalence}, which describes how features and classifiers converge to simplex ETF geometry during training. Subsequent theoretical work \cite{zhu2021geometric,yaras2022neural,thrampoulidis2022imbalance,gill2024engineering} has characterized neural collapse under various conditions. Our analysis differs in scope: rather than explaining why training produces simplex geometry, we provide guarantees for the \emph{downstream open-set recognition} task that simplex-aligned embeddings enable.

\textbf{Medical imaging OSR.}
OSR applications in medical imaging have grown in importance, with work exploring MedMNIST benchmarks \cite{yang2021medmnist}, evidential uncertainty quantification \cite{wang2023uncertainty}, and distillation-based improvements \cite{jia2024revealing}. However, no prior work establishes theoretical guarantees specific to simplex-based OSR in medical or clinical settings.

%% ============================================================
%% 3. PRELIMINARIES AND PROBLEM SETUP
%% ============================================================
\section{Preliminaries}\label{sec:prelim}

We state the formal assumptions underpinning our framework and fix notation used throughout.

\noindent\textbf{Notation and score variants.}
We write $[C]:=\{1,\dots,C\}$, $[x]_+:=\max(x,0)$, and use $\Prob$ and $\E$ for probability and expectation. Both UCDSC training and the reject rule operate on the unsquared distance-ratio score $U(z)=d_{\min}(z)/\mu(z)$. We show that $U$ is globally Lipschitz for \emph{any} distinct prototypes (Section~\ref{sec:concentration}). For geometric tractability, we additionally introduce a squared variant $U_2(z)$ (Assumption~\ref{ass:scores}) whose sublevel sets admit closed-form descriptions under the balanced equal-norm condition A1$'$ (Section~\ref{sec:geometry}). Proposition~\ref{prop:global-comparison} establishes the global lower bound $\sqrt{U_2(z)}\le U(z)$. Under the balanced equal-norm condition A1$'$, Theorem~\ref{thm:geometry-U2}(iv) yields the global sandwich $\{U_2\le \theta^2/2\}\subseteq\{U\le\theta\}\subseteq\{U_2\le\theta^2\}$; under the simplex (A1$''$), Lemma~\ref{lem:rms-am} improves the inner constant to $3\theta^2/4$. Near each prototype, Theorem~\ref{thm:local-comparison} further sharpens the comparison with explicit dependence on the simplex-defect parameter $\lambda_j$.

\begin{assumption}[A0: Basic objects]\label{ass:basic}
Let $C\ge 2$ be the number of known classes and $d\ge 2$ the embedding dimension.
Let $f:\cX\to\R^d$ be the learned embedding.
The known distribution is $P_K$ on $(X,Y)\in\cX\times[C]$, with $Z:=f(X)$.
The unknown distribution is $P_U$ on $X\in\cX$, with $Z_U:=f(X)$.
The auxiliary background distribution is $P_B$ on $X_B\in\cX$, with $Z_B:=f(X_B)$.
\end{assumption}

We introduce a hierarchy of prototype assumptions of increasing strength.
The weakest (A1) requires only distinct prototypes and suffices for Lipschitz regularity and compact acceptance regions.
The intermediate (A1$'$) adds balanced equal-norm structure and enables exact $U_2$ ball geometry.
The strongest (A1$''$) recovers the regular simplex as a special case.

\begin{assumption}[A1: Distinct prototypes]\label{ass:distinct}
Prototypes $\{s_1,\dots,s_C\}\subset\R^d$ are distinct. Define the geometric statistics:
\begin{align}
A_j&:=\frac{1}{C-1}\sum_{k\neq j}|s_j-s_k|, & B_j&:=\frac{1}{C-1}\sum_{k\neq j}|s_j-s_k|^2,\label{eq:Aj-Bj}\\
\Delta_j&:=\min_{k\neq j}|s_j-s_k|, & \lambda_j&:=\frac{\sqrt{B_j}}{A_j}\ge 1.\label{eq:Deltaj-lambdaj}
\end{align}
Set $A_{\min}:=\min_j A_j$, $A_{\max}:=\max_j A_j$, $\Delta:=\min_j\Delta_j$.
We call $\lambda_j$ the \emph{simplex-defect parameter} of class $j$: it equals $1$ if and only if all rival distances from $s_j$ are equal.
\end{assumption}

\begin{assumption}[A1$'$: Balanced equal-norm code]\label{ass:balanced}
The prototypes satisfy A1 and additionally:
\begin{equation}\label{eq:balanced}
\|s_j\|=R\quad\forall j,\qquad \sum_{j=1}^C s_j=0,
\end{equation}
for some $R>0$.
Such codes exist in every dimension $d\ge 2$ for every $C\ge 3$: a regular $C$-gon in a $2$-plane embedded in $\R^d$ suffices (see Section~\ref{sec:polygon}).
\end{assumption}

\begin{assumption}[A1$''$: Simplex-ETF prototypes]\label{ass:simplex}
The prototypes satisfy A1$'$ and additionally form a regular simplex:
\begin{equation}\label{eq:simplex}
\langle s_i,s_j\rangle=-\frac{R^2}{C-1}\quad(i\neq j).
\end{equation}
The common edge length is $D:=\|s_i-s_j\|=R\sqrt{2C/(C-1)}$.
This requires $d\ge C-1$.
Under A1$''$, $A_j=D$, $B_j=D^2$, $\Delta_j=D$, and $\lambda_j=1$ for all $j$.
\end{assumption}

\begin{assumption}[A2: Scores]\label{ass:scores}
For $z\in\R^d$, let $d_j(z):=\|z-s_j\|$, $d_{\min}(z):=\min_j d_j(z)$, $j^*(z):=\arg\min_j d_j(z)$ (with ties broken by choosing the smallest index), and
\begin{subequations}\label{eq:scores}
\begin{align}
\mu(z)&:=\frac{1}{C-1}\sum_{k\neq j^*(z)} d_k(z),\qquad
U(z):=\frac{d_{\min}(z)}{\mu(z)},\label{eq:mu}\\
\nu(z)&:=\frac{1}{C-1}\!\left(\sum_j\|z-s_j\|^2-\min_j\|z-s_j\|^2\right),\label{eq:nu}\\
U_2(z)&:=\frac{\min_j\|z-s_j\|^2}{\nu(z)},\qquad
\alpha(z):=\max_j\langle z,s_j\rangle.\label{eq:U2-alpha}
\end{align}
\end{subequations}
Note that $U_2$ is a distinct score, not the square of $U$.
\end{assumption}

\noindent\textbf{Tie-breaking.}
When the minimizer is not unique, we define $j^*(z):=\min\arg\min_{j\in[C]} d_j(z)$ (smallest index among minimizers).
\emph{Remark:} ties occur only on a finite union of bisector hyperplanes and hence on a Lebesgue-null set; this convention does not affect any probabilistic statements under absolutely continuous distributions (which includes all models considered in this paper).

\begin{assumption}[A3: Reject rule]\label{ass:reject}
Accept as known iff $U(z)\le\theta$ for threshold $\theta\in[0,1]$.
The closed-set label is $\hat y(z):=j^*(z)$.
\end{assumption}

\begin{assumption}[A4: UCDSC loss]\label{ass:loss}
Given training data $\{(x_i,y_i)\}_{i=1}^n$ with $y_i\in[C]$ and background samples $\{x_k^{\mathrm{bg}}\}_{k=1}^K$, define $f_i:=f(x_i)$ and $f_k^{\mathrm{bg}}:=f(x_k^{\mathrm{bg}})$.
Let $m>0$ be a margin parameter and $\lambda_o,\lambda_u\ge 0$ be loss weights.
The total loss is $L_{\mathrm{total}}=L_{\mathrm{intra}}+\lambda_o L_o+\lambda_u L_u$ where:
\begin{align}
L_{\mathrm{intra}}&:=\frac{1}{n}\sum_{i=1}^n\|f_i-s_{y_i}\|^2, \label{eq:Lintra}\\
L_o&:=\sum_{i=1}^n\sum_{k=1}^K\max\!\bigl(0,m+\|f_i-s_{y_i}\|^2-\|f_k^{\mathrm{bg}}-s_{y_i}\|^2\bigr),\label{eq:Lo}\\
L_u&:=\frac{1}{n}\sum_{i=1}^n U(f_i).\label{eq:Lu}
\end{align}
\textbf{Norm convention.}
Note that $U$ in~\eqref{eq:mu} (and hence $L_u$) uses \emph{unsquared} Euclidean norms $\|z-s_j\|$, matching the UCDSC definition \cite{aditya2025ucdsc}, whereas $L_{\mathrm{intra}}$~\eqref{eq:Lintra} and $L_o$~\eqref{eq:Lo} use squared norms $\|z-s_j\|^2$.
\end{assumption}

\begin{assumption}[A5: One-sided score-tail control]\label{ass:conc}
For each class $y$, the uncertainty score $U(Z)\mid Y\!=\!y$ satisfies a one-sided sub-Gaussian upper-tail inequality about its mean with scale $\tau_y$: for all $t\ge 0$,
\begin{equation}\label{eq:concentration}
\Prob\!\left(U(Z)-\E[U(Z)\mid Y\!=\!y]\ge t\mid Y\!=\!y\right)
\le
\exp\!\left(-\frac{t^2}{2\tau_y^2}\right),
\end{equation}
with $\tau:=\max_y\tau_y$.
We allow the limiting case $\tau_y=0$, interpreted as zero upper-tail mass above the mean.
We define $\sigma_y:=\tau_y/L_U$ where $L_U$ is the Lipschitz constant of $U$: under A1 alone one may take $L_U=4/\Delta$; under A1$''$ (simplex), $\Delta=D$ and hence $L_U=4/D$.
Set $\sigma:=\max_y\sigma_y$, so that $\tau=L_U\sigma$.
Since $U(Z)\in[0,1]$, Hoeffding's lemma yields~\eqref{eq:concentration} with the universal choice $\tau_y=1/2$.
In the results below, A5 is used only through whatever sharper class-specific values of $\tau_y$ are separately assumed or validated from data.
\end{assumption}

\begin{remark}[Scope of A5]\label{rem:a5-scope}
A5 is a condition on the \emph{scalar} score $U(Z)\mid Y\!=\!y$.
The paper does not claim that simplex geometry or UCDSC training alone implies A5.
If one has an independent feature-space concentration theorem for $Z\mid Y\!=\!y$ saying that every Euclidean $L$-Lipschitz observable $g$ satisfies
\[
\Prob(g(Z)-\E[g(Z)\mid Y\!=\!y]\ge t\mid Y\!=\!y)
\le \exp\!\left(-\frac{t^2}{2L^2\rho_y^2}\right),
\]
for a class-conditional feature-space scale $\rho_y$, then applying it to $g=U$ (with $L=L_U$) yields~\eqref{eq:concentration} with $\tau_y=L_U\rho_y$.
This external scale $\rho_y$ is distinct from the rescaled score-tail parameter $\sigma_y:=\tau_y/L_U$ unless such a feature-space theorem is actually being invoked.
We do not rely on any specific such theorem here.
\end{remark}

\begin{assumption}[A6: Unknown features]\label{ass:unknown}
Two analysis routes, each specifying a pre-processing map $\Phi$:
\textbf{(U-Norm):} $\Phi(z)=Rz/\|z\|$ for $z\neq 0$, and $\Phi(0):=0$ by convention (this is immaterial for the FAR bounds below, which are stated for $\theta<1$, because then $U(0)=1$ and the origin is rejected); assume that for all unit vectors $v\in S^{d-1}$ and all $t\in[0,1]$, $|\Prob(\langle\Phi(Z_U)/R,v\rangle\ge t\mid Z_U\neq 0)-\pcap(t;d)|\le\delta_{\mathrm{cap}}$, where $\delta_{\mathrm{cap}}\ge 0$ is a fixed constant.
\textbf{(U-RadIso):} $\Phi=\mathrm{id}$ (no normalization); write $Z_U=R_UV_U$ with $R_U:=\|Z_U\|$ and $V_U\mid R_U$ uniform on $S^{d-1}$.
\end{assumption}

\begin{assumption}[A6$'$: Cap tail]\label{ass:cap}
$\pcap(t;d):=\Prob(\langle W,v\rangle\ge t)\le\exp(-(d-1)t^2/2)$ for any fixed $v\in S^{d-1}$, $t\in[0,1]$, and $W\sim\mathrm{Unif}(S^{d-1})$.
The quantity is independent of $v$ by rotational symmetry.
\end{assumption}

We define the operating metrics (where $\Phi$ is specified by the chosen model in A6):
\begin{equation}\label{eq:metrics}
\FAR(\theta):=\Prob(U(\Phi(Z_U))\le\theta),\quad
\KRR(\theta):=\Prob(U(Z)>\theta).
\end{equation}

%% ============================================================
%% 4. LIPSCHITZ ANALYSIS AND CONCENTRATION
%% ============================================================
\section{Lipschitz Analysis, Ball Control, and Concentration}\label{sec:concentration}

We begin by establishing results for \emph{arbitrary distinct} prototypes (A1 only), requiring no simplex or balanced-code structure. We then record the simplex (A1$''$) specialization, where $\Delta=D$.

\begin{theorem}[Generic ball control and Lipschitz stability]\label{thm:generic-ball}
Assume A1 (distinct prototypes) and A2, with $C\ge 2$. Then for every $\theta\in(0,1)$:
\[
\bigcup_{j=1}^C B\!\left(s_j,\; r_j^-\right)\subseteq \{U\le\theta\} \subseteq \bigcup_{j=1}^C \overline B\!\left(s_j,\; r_j^+\right),
\]
where
\[
r_j^-:=\min\!\left\{\frac{\Delta_j}{2},\;\frac{\theta A_j}{1+\theta}\right\},\qquad
r_j^+:=\frac{\theta A_j}{1-\theta}.
\]
In particular $\{U\le\theta\}$ is compact.
Moreover, $U$ is globally Lipschitz:
\begin{equation}\label{eq:lip-general}
|U(z)-U(w)|\le \frac{4}{\Delta}\|z-w\|,\qquad \Delta:=\min_{i\neq j}|s_i-s_j|.
\end{equation}
\end{theorem}

\begin{proof}
Let $z\in\R^d$, $j=j^*(z)$, $r=d_j(z)=d_{\min}(z)$.

\emph{Outer bound.}
For every $k\neq j$, triangle inequality gives $d_k(z)\le r+|s_j-s_k|$. Averaging over $k\neq j$: $\mu(z)\le r+A_j$. Hence $U(z)=r/\mu(z)\ge r/(r+A_j)$.
So if $U(z)\le\theta$, then $r/(r+A_j)\le\theta$, giving $r\le \theta A_j/(1-\theta)=r_j^+$.

\emph{Inner bound.}
Fix $j$ and suppose $|z-s_j|<\Delta_j/2$. Then for every $k\neq j$, $d_k(z)\ge |s_k-s_j|-|z-s_j|>\Delta_j/2>|z-s_j|$, so $j$ is the unique nearest prototype. Therefore $\mu(z)\ge A_j-|z-s_j|$, hence $U(z)\le |z-s_j|/(A_j-|z-s_j|)$. If $|z-s_j|\le\theta A_j/(1+\theta)$, the right side is at most $\theta$.

\emph{Lipschitz bound.}
$d_{\min}$ is $1$-Lipschitz.
$\mu(z)=\frac{1}{C-1}(\sum_m d_m(z)-d_{\min}(z))$, which is independent of how ties for $j^*(z)$ are broken.
Equivalently, if $\mu_j(z):=(C-1)^{-1}\sum_{k\neq j}d_k(z)$, then $\mu(z)=\max_j\mu_j(z)$ because removing a nearest prototype leaves the largest rival average.
Each $\mu_j$ is $1$-Lipschitz, so $\mu$ is $1$-Lipschitz.
For any $k\neq j^*(z)$: $d_k(z)+d_{\min}(z)\ge |s_k-s_{j^*}|\ge \Delta$, so $d_k(z)\ge \Delta/2$, hence $\mu(z)\ge \Delta/2$.
The quotient rule gives~\eqref{eq:lip-general}.
The outer inclusion gives boundedness of $\{U\le\theta\}$, and Lipschitz continuity gives closedness; hence the acceptance region is compact.
\end{proof}

Under the simplex assumption A1$''$, this becomes the following specialization:

\begin{lemma}[Lipschitz continuity under A1$''$]\label{lem:lipschitz}
Under A1$''$--A2:
(i) $d_k(z)\ge D/2$ for all $k\neq j^*(z)$; in particular, $\mu(z)\ge D/2$;
(ii) $|U(z)-U(z')|\le(4/D)\|z-z'\|$ for all $z,z'$.
\end{lemma}

\begin{proof}[Proof sketch]
Under A1$''$, all pairwise distances equal $D$, so part (i) follows from the triangle inequality as above.
The identity $\mu=\max_j\mu_j$ (where $\mu_j:=\frac{1}{C-1}\sum_{k\neq j}d_k$) is tie-independent and holds for arbitrary distinct prototypes; in the simplex case $\Delta=D$, so the generic Lipschitz bound gives $L_U=4/D$.
Full proof in Appendix~\ref{app:lipschitz}.
\end{proof}

\begin{remark}[Tightness]\label{rem:tight}
For $C=2$, the bound $4/D$ is tight.
Let $z(r)=s_1+r(s_2-s_1)/D$ for $0\le r\le D/2$.
Then $U(z(r))=r/(D-r)$, and the one-sided derivative $D/(D-r)^2$ tends to $4/D$ at the midpoint.
Thus the constant is attained in the two-prototype case by a difference quotient approaching the midpoint.
The representation $\mu=\max_j\mu_j$ used above is not simplex-specific; it is a deterministic consequence of removing a nearest prototype from the average.
\end{remark}

\noindent\textbf{Quantile convention.}
For a real-valued random variable $W$, write
\[
q_{1-\varepsilon}(W):=\inf\{t\in\R:\Prob(W\le t)\ge 1-\varepsilon\}.
\]
Conditional quantiles are defined analogously.

\begin{lemma}[Known-class quantile control]\label{lem:quantile}
Under A1--A2 and A5: for each class $y$ and $\varepsilon\in(0,1)$,
\begin{equation}\label{eq:quantile}
q_{1-\varepsilon}(U(Z)\mid Y\!=\!y)\le\E[U(Z)\mid Y\!=\!y]+\tau_y\sqrt{2\log(1/\varepsilon)}.
\end{equation}
In particular, using $\tau_y=L_U\sigma_y$ and $\tau=L_U\sigma$:
\begin{align*}
q_{1-\varepsilon}(U(Z)\mid Y\!=\!y)
&\le\E[U(Z)\mid Y\!=\!y]+L_U\sigma_y\sqrt{2\log(1/\varepsilon)}\\
&\le\E[U(Z)\mid Y\!=\!y]+L_U\sigma\sqrt{2\log(1/\varepsilon)}.
\end{align*}
\end{lemma}

\begin{proof}
Invert the A5 tail bound~\eqref{eq:concentration} at level $\varepsilon$.
\end{proof}

%% ============================================================
%% 5. EXACT U2 GEOMETRY AND BALL-SANDWICH BOUNDS
%% ============================================================
\section{Exact $U_2$ Geometry and Bounds for $U$}\label{sec:geometry}

We now show that the auxiliary score $U_2$ has an exact union-of-balls geometry under the balanced equal-norm condition (A1$'$), and that the operational acceptance region $\{z:U(z)\le\theta\}$ lies between corresponding unions of Euclidean balls. We first establish a global comparison between $U$ and $U_2$ that holds for \emph{all} prototype configurations.

\begin{proposition}[Global lower comparison]\label{prop:global-comparison}
Under A1--A2 (any distinct prototypes), for every $z$:
\begin{equation}\label{eq:global-lower}
\sqrt{U_2(z)}\le U(z).
\end{equation}
Hence $\{U\le\theta\}\subseteq\{U_2\le\theta^2\}$ for every $\theta\ge 0$.
\end{proposition}

\begin{proof}
Let $j=j^*(z)$, $m=d_j(z)$, and $a_k=d_k(z)$ for $k\neq j$.
Since the prototypes are distinct and $C\ge 2$, the rival distances relative to any nearest prototype cannot all vanish; hence $\mathrm{AM}>0$ and $\mathrm{RMS}>0$ below.
If $m=0$, then $\sqrt{U_2(z)}=U(z)=0$, so the claim is immediate.
Write
\[
\begin{aligned}
\mathrm{AM}&:=\frac{1}{C-1}\sum_{k\neq j} a_k=\mu(z),\\
\mathrm{RMS}&:=\sqrt{\frac{1}{C-1}\sum_{k\neq j} a_k^2}=\sqrt{\nu(z)}.
\end{aligned}
\]
Then $U(z)=m/\mathrm{AM}$ and $\sqrt{U_2(z)}=m/\mathrm{RMS}$, while $\mathrm{AM}\le\mathrm{RMS}$ by Cauchy--Schwarz.
The argument is unaffected by ties because both $\mu$ and $\nu$ remove one nearest distance and all nearest distances have the same value.
If $U(z)\le\theta$, then $\sqrt{U_2(z)}\le U(z)\le\theta$, hence $U_2(z)\le\theta^2$.
\end{proof}

Under the simplex (A1$''$), we additionally have an \emph{upper} comparison:

\begin{lemma}[$U$--$U_2$ equivalence under A1$''$]\label{lem:rms-am}
Under A1$''$--A2, for all $z\in\R^d$:
\begin{equation}\label{eq:sandwich-U}
\sqrt{U_2(z)}\le U(z)\le \frac{2}{\sqrt{3}}\sqrt{U_2(z)},
\end{equation}
and hence $\{U_2\le 3\theta^2/4\}\subseteq\{U\le\theta\}\subseteq\{U_2\le\theta^2\}$ for every $\theta\ge 0$.
\end{lemma}

\begin{proof}[Proof sketch]
The lower bound is Proposition~\ref{prop:global-comparison}.
If $d_{\min}(z)=0$, the upper bound is also immediate.
For the upper bound, Lemma~\ref{lem:lipschitz}(i) gives $\min_k a_k\ge D/2$, and triangle inequality gives $a_k\le d_{\min}+D\le 3\min_k a_k$, so $\max/\min\le 3$.
Among vectors in $[1,3]$, $\mathrm{RMS}/\mathrm{AM}\le 2/\sqrt{3}$ (achieved at a two-point distribution on $\{1,3\}$).
Full derivation in Appendix~\ref{app:geometry}.
\end{proof}

\begin{theorem}[Euclidean-ball geometry of $U_2$]\label{thm:geometry-U2}
Under A1$'$--A2 (balanced equal-norm code), for $\rho\in[0,1)$, define $t(\rho):=(1-\rho)/(1+\rho/(C-1))$ and $\gamma_\rho:=1/t(\rho)$, with $t(0)=\gamma_0=1$.
Then:
\begin{enumerate}[label=(\roman*),leftmargin=1.8em,itemsep=1pt]
\item $\displaystyle U_2(z)=\frac{\|z\|^2+R^2-2\alpha(z)}{\|z\|^2+R^2+\frac{2}{C-1}\alpha(z)}$.

\item $\{z:U_2(z)\le\rho\}=\bigcup_{j=1}^C\mathbb B_j(\rho)$ where $\mathbb B_j(\rho)$ is the closed ball
\begin{equation}\label{eq:ball}
\mathbb B_j(\rho):=\!\left\{z:\left\|z-\gamma_\rho s_j\right\|^2\!\le R^2\!\left(\gamma_\rho^2-1\right)\right\}.
\end{equation}
This set is compact for $\rho<1$.

\item On $\|z\|=R$: $U_2(z)\le\rho\iff\exists j:\langle z,s_j\rangle\ge t(\rho)R^2$.

\item For $\theta\in[0,1)$,
\[
\{U_2\le\theta^2/2\}\subseteq\{U\le\theta\}\subseteq\{U_2\le\theta^2\},
\]
i.e., $\bigcup_j\mathbb B_j(\theta^2/2)\subseteq\{U\le\theta\}\subseteq\bigcup_j\mathbb B_j(\theta^2)$.
Under A1$''$ (simplex), the inner set tightens to $\bigcup_j\mathbb B_j(3\theta^2/4)$ via Lemma~\ref{lem:rms-am}.
\end{enumerate}
\end{theorem}

\begin{proof}[Proof sketch]
Part (i): expand $\|z-s_j\|^2=\|z\|^2+R^2-2\langle z,s_j\rangle$ and use $\sum_j s_j=0$ (A1$'$).
Part (ii): for $\rho=0$, $U_2(z)=0$ iff $z$ equals one of the prototypes, which matches $\mathbb B_j(0)=\{s_j\}$.
For $\rho\in(0,1)$, rearrange $U_2\le\rho$ to $\alpha(z)\ge\frac{t(\rho)}{2}(\|z\|^2+R^2)$; complete the square for each $j$.
Part (iii): substitute $\|z\|=R$.
Part (iv): the outer inclusion is Proposition~\ref{prop:global-comparison}. For the balanced-code inner inclusion, write $r=\|z\|$ and let $a_k=d_k(z)$ for $k\neq j^*(z)$. Then $a_k\le r+R$, so $\mathrm{RMS}^2\le(r+R)\mathrm{AM}$, while part~(i) gives $\mathrm{RMS}^2=\nu(z)\ge r^2+R^2$ because $\alpha(z)\ge 0$. Hence $\mathrm{RMS}/\mathrm{AM}\le (r+R)/\sqrt{r^2+R^2}\le\sqrt{2}$, i.e.\ $U\le \sqrt{2}\sqrt{U_2}$. The simplex sharpening uses Lemma~\ref{lem:rms-am}.
Full details in Appendix~\ref{app:geometry}.
\end{proof}

\begin{corollary}[Tangent-ball separation diagnostic]\label{cor:tau-sep}
Assume A1$'$--A2 and let
\[
\delta_{\min}:=\min_{i\ne j}\|s_i-s_j\|.
\]
For $\delta_{\min}<2R$, define
\[
\gamma_*:=\frac{1}{\sqrt{1-\delta_{\min}^2/(4R^2)}},
\qquad
\tau_{\mathrm{sep}}:=
\frac{(C-1)(\gamma_*-1)}{1+(C-1)\gamma_*}.
\]
When $\delta_{\min}=2R$, use the limiting convention $\tau_{\mathrm{sep}}=1$.
Then the balls $\mathbb B_j(\rho)$ in Theorem~\ref{thm:geometry-U2}(ii) have disjoint interiors for every $\rho\in[0,\tau_{\mathrm{sep}}]$ and are disjoint as closed sets for every $\rho<\tau_{\mathrm{sep}}$.
If a pair of prototypes attains $\delta_{\min}$, the corresponding two balls are tangent at $\rho=\tau_{\mathrm{sep}}$.
\end{corollary}

\begin{proof}
For fixed $\rho$, the centers of $\mathbb B_i(\rho)$ and $\mathbb B_j(\rho)$ are separated by $\gamma_\rho\|s_i-s_j\|$ and each ball has radius $R\sqrt{\gamma_\rho^2-1}$.
Thus the first contact occurs when
\[
\gamma_\rho\delta_{\min}=2R\sqrt{\gamma_\rho^2-1}.
\]
Solving gives $\gamma_\rho=\gamma_*$.
Since $\gamma_\rho=(1+\rho/(C-1))/(1-\rho)$ is increasing in $\rho$, inverting this equality yields the displayed expression for $\tau_{\mathrm{sep}}$.
\end{proof}

\textbf{Key point:} Theorem~\ref{thm:geometry-U2} requires only the balanced equal-norm condition A1$'$, \emph{not} the full simplex A1$''$. The proof uses $\|s_j\|=R$ and $\sum_j s_j=0$; equiangularity is never invoked. This means exact $U_2$-ball geometry is available in \emph{every} dimension $d\ge 2$, including $d<C-1$.

This theorem reveals that the $U_2$-sublevel set $\{U_2\le\rho\}$ is a union of $C$ (possibly overlapping) closed Euclidean balls, one per prototype.
By the sandwich (Theorem~\ref{thm:geometry-U2}(iv)), the operational $U$-acceptance region $\{U\le\theta\}$ lies between two compact prototype-aligned unions of balls. In the balanced equal-norm setting the global inner transfer uses $\theta^2/2$, while the simplex special case improves this to $3\theta^2/4$.
The ball centers $\gamma_\rho s_j$ lie along the prototype directions at distance $\gamma_\rho R>R$ from the origin, and the ball radii grow as $\rho\to 1$.
Combining with Theorem~\ref{thm:generic-ball}, this gives fully explicit compactness for $U$-acceptance regions for any balanced equal-norm code:

\begin{corollary}[Compactness for $U$ via $U_2$ transfer]\label{cor:compact-U}
Under A1$'$--A2, for any $\theta\in(0,1)$:
\[
\{U\le\theta\}\subseteq \bigcup_{j=1}^C \overline B\!\left(\gamma_{\theta^2} s_j,\;R\sqrt{\gamma_{\theta^2}^2-1}\right).
\]
Consequently $\{U\le\theta\}$ is compact: it is closed by the Lipschitz continuity of $U$ and bounded by the finite compact outer cover above.
\end{corollary}

\begin{figure}[t]
\centering
\includegraphics[width=\columnwidth]{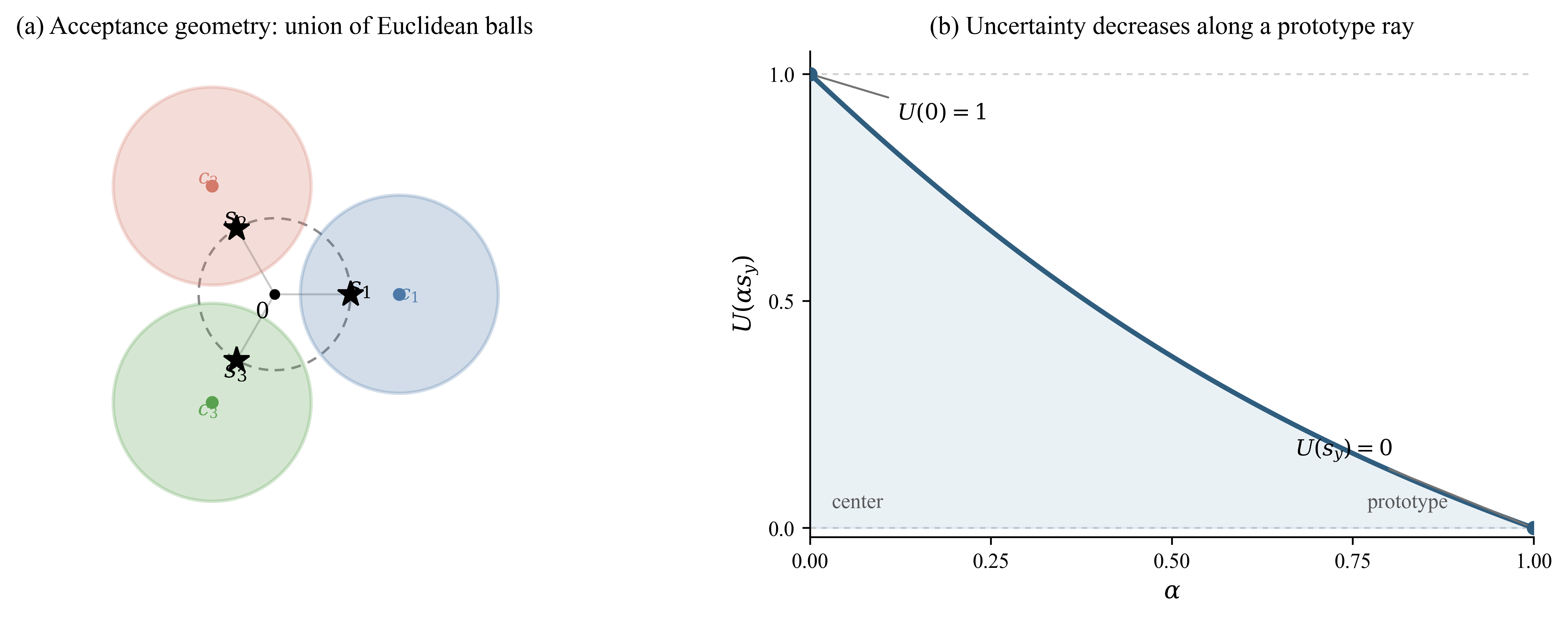}
\caption{Illustration of the geometry for $C=3$. Left: the $U_2$-sublevel set is a union of Euclidean balls $\mathbb B_j(\rho)$ centered along prototype directions, as characterized in Theorem~\ref{thm:geometry-U2}(ii). Right: along a prototype ray $z=\alpha s_y$, the uncertainty decreases monotonically from the simplex center to the class prototype.}
\label{fig:union-balls}
\end{figure}

%% ============================================================
%% 5.1. SHARP DICHOTOMY AND DEFECT PARAMETER
%% ============================================================
\subsection{Sharp Dichotomy and the Simplex-Defect Parameter}\label{sec:dichotomy}

The simplex-defect parameter $\lambda_j=\sqrt{B_j}/A_j$ introduced in A1 quantifies how far each class is from the one-distance simplex geometry. The following theorem shows that $d\ge C-1$ is the phase transition for $\lambda_j=1$, not for the rejection geometry itself.

\begin{theorem}[Sharp dichotomy]\label{thm:dichotomy}
Assume A1$'$ (balanced equal-norm). The following are equivalent:
\begin{enumerate}[label=(\roman*),leftmargin=1.8em,itemsep=1pt]
\item $\lambda_j=1$ for every $j$;
\item all pairwise distances $|s_i-s_j|$ are equal;
\item the prototypes form a regular simplex.
\end{enumerate}
Consequently, if $C>d+1$, then necessarily $\max_j\lambda_j>1$.
\end{theorem}

\begin{proof}
For each fixed $j$, $\lambda_j=\sqrt{B_j}/A_j\ge 1$ is the RMS/AM ratio of the positive rival distances $\{|s_j-s_k|:k\neq j\}$.
Equality holds iff all rival distances from $s_j$ are equal.
Thus (ii) implies (i), and (i) implies that each row of the off-diagonal distance matrix is constant.
By symmetry of the distance matrix, if row $j$ has common value $r_j$, then the entry $|s_i-s_j|$ equals both $r_i$ and $r_j$ for every $i\neq j$; hence all $r_j$ are equal and all pairwise distances are equal, proving (ii).
Under A1$'$, if all pairwise distances are equal to $D$, then
\[
D^2=\|s_i-s_j\|^2=2R^2-2\langle s_i,s_j\rangle
\]
for every $i\neq j$, so all off-diagonal inner products are common.
Since $\sum_j s_j=0$, taking the inner product with $s_i$ gives $R^2+\sum_{j\neq i}\langle s_i,s_j\rangle=0$, hence the common value is $-R^2/(C-1)$.
This is exactly the regular simplex condition A1$''$, proving (iii).
The implication (iii) $\Rightarrow$ (ii) is immediate from A1$''$.
Finally, $C$ equidistant points have affine dimension $C-1$, so such prototypes can exist in $\R^d$ only when $d\ge C-1$; therefore if $C>d+1$, no balanced equal-norm code can have all $\lambda_j=1$.
\end{proof}

\begin{theorem}[Local equivalence of $U$ and $\sqrt{U_2}$]\label{thm:local-comparison}
Assume A1--A2. If $z\in B(s_j,r)$ with $r<\Delta_j/2$, then $j=j^*(z)$ and
\begin{equation}\label{eq:local-equiv}
\sqrt{U_2(z)}\le U(z)\le \frac{\sqrt{B_j}+r}{A_j-r}\sqrt{U_2(z)}.
\end{equation}
In particular, if $r\le \eta A_j$ with $0<\eta<\Delta_j/(2A_j)$:
\[
\sqrt{U_2(z)}\le U(z)\le \frac{\lambda_j+\eta}{1-\eta}\sqrt{U_2(z)}.
\]
\end{theorem}

\begin{proof}
Let $\rho:=\|z-s_j\|\le r$.
Because $r<\Delta_j/2$, $j$ is uniquely nearest (as in Theorem~\ref{thm:generic-ball}).
If $\rho=0$, then $U(z)=\sqrt{U_2(z)}=0$ and the claim is immediate.
For rival distances $a_k:=d_k(z)$, $k\neq j$, define
\[
\mathrm{AM}:=\frac{1}{C-1}\sum_{k\neq j}a_k,\qquad
\mathrm{RMS}:=\sqrt{\frac{1}{C-1}\sum_{k\neq j}a_k^2}.
\]
The triangle inequality gives $\mathrm{AM}\ge A_j-\rho$, while Minkowski gives $\mathrm{RMS}\le\sqrt{B_j}+\rho$.
Since
\[
\frac{U(z)}{\sqrt{U_2(z)}}=\frac{\mathrm{RMS}}{\mathrm{AM}},
\]
we obtain
\[
U(z)\le \frac{\sqrt{B_j}+\rho}{A_j-\rho}\sqrt{U_2(z)}
\le \frac{\sqrt{B_j}+r}{A_j-r}\sqrt{U_2(z)},
\]
where the last step uses monotonicity in $\rho<A_j$.
The final display follows from $r\le \eta A_j$ and $\lambda_j=\sqrt{B_j}/A_j$.
\end{proof}

\textbf{Interpretation.}
Near prototype $s_j$, the scores $U$ and $\sqrt{U_2}$ are equivalent up to a factor that depends on $\lambda_j$. For $\lambda_j=1$ (simplex case), the local comparison factor approaches $1$ as $r\to 0$. For $\lambda_j>1$ (low-dimensional regime), the comparison worsens with explicit constants.

%% ============================================================
%% 5.2. CONCRETE LOW-DIMENSIONAL CONSTRUCTION
%% ============================================================
\subsection{Universal Low-Dimensional Construction: The Regular $C$-gon}\label{sec:polygon}

For every $d\ge 2$ and $C\ge 3$, the regular $C$-gon provides a concrete balanced equal-norm code:
\begin{equation}\label{eq:polygon}
s_j = R\bigl(\cos(2\pi j/C),\;\sin(2\pi j/C),\;0,\dots,0\bigr),\quad j=1,\dots,C.
\end{equation}
Indexing by $j=1,\dots,C$ gives the same roots of unity as indexing by $0,\dots,C-1$.
This satisfies $|s_j|=R$ and $\sum_j s_j=0$. By symmetry, $A_j=A$, $B_j=B$, $\Delta_j=\Delta$, $\lambda_j=\lambda$ are all independent of $j$, with:
\begin{align*}
\Delta &= 2R\sin\frac{\pi}{C},\quad
A = \frac{2R}{C-1}\cot\frac{\pi}{2C},\\
B &= \frac{2CR^2}{C-1},\quad
\lambda = \frac{\sqrt{2C(C-1)}}{2\cot(\pi/(2C))}.
\end{align*}
These formulas use the standard trigonometric sums
\[
\sum_{n=1}^{C-1}\sin(\pi n/C)=\cot(\pi/(2C)),\qquad
\sum_{n=1}^{C-1}\sin^2(\pi n/C)=C/2.
\]
The regular $C$-gon is the simplest exposition of the construction: it uses the first Fourier frequency and then pads by zeros.
Our experiments use the standard finite harmonic-frame generalisation, implemented by stacking the first $\lfloor d/2\rfloor$ cosine--sine Fourier frequency pairs and normalising to radius $R$ (with one zero padding coordinate when $d$ is odd).
This multi-frequency Naimark/DFT-truncation construction preserves the balanced equal-norm conditions of A1$'$ while using the available embedding dimension more effectively; the $C$-gon in \eqref{eq:polygon} is precisely the first-frequency special case.
This is the same finite-harmonic-code viewpoint used in harmonic equiangular tight-frame constructions~\cite{fickus2019harmonic}.

For any fixed $\theta\in(0,1)$, define
\[
\begin{aligned}
r^-_\theta&:=\min\!\left\{R\sin\frac{\pi}{C},\;
\frac{2\theta R}{(1+\theta)(C-1)}\cot\frac{\pi}{2C}\right\},\\
r^+_\theta&:=\frac{2\theta R}{(1-\theta)(C-1)}\cot\frac{\pi}{2C}.
\end{aligned}
\]
Theorem~\ref{thm:generic-ball} gives
\[
\bigcup_{j=1}^C B(s_j,r^-_\theta)
\subseteq \{U\le\theta\}
\subseteq \bigcup_{j=1}^C \overline B(s_j,r^+_\theta).
\]
At the same time, for $\rho\in[0,1)$, the exact $U_2$ ball geometry from Theorem~\ref{thm:geometry-U2} gives $\{U_2\le\rho\}=\bigcup_j\overline B(\gamma_\rho s_j, R\sqrt{\gamma_\rho^2-1})$.

For fixed $d=2$ as $C\to\infty$, $A\sim 4R/\pi$ and $\Delta\sim 2\pi R/C$. For fixed $\theta<1$, the guaranteed confident cores shrink because packing gets tighter, but the acceptance region remains compact with $O_\theta(R)$ outer scale, giving controlled degradation rather than collapse.

%% ============================================================
%% 6. FALSE ACCEPTANCE RATE BOUNDS
%% ============================================================
\section{False Acceptance Rate Bounds}\label{sec:far}

\begin{theorem}[KRR and conditional FAR bounds]\label{thm:far-krr}
Under A1$'$--A2 (balanced equal-norm code), and with additional assumptions as specified in each part:

\noindent\textbf{(i) KRR control.}
Assume A5 and let $\varepsilon\in(0,1)$.
Setting
\[
\theta_\varepsilon:=\min\!\left\{1,\max_y\left(\E[U(Z)\mid Y\!=\!y]+\tau_y\sqrt{2\log(1/\varepsilon)}\right)\right\}
\]
yields $\KRR(\theta_\varepsilon)\le\varepsilon$.
In particular, the conservative common-scale choice $\theta_\varepsilon=\min\!\left\{1,\max_y\left(\E[U(Z)\mid Y\!=\!y]+L_U\sigma\sqrt{2\log(1/\varepsilon)}\right)\right\}$ also works.

\noindent\textbf{(ii) FAR via geometry (for $\theta\in[0,1)$).}
$\FAR(\theta)\le\Prob(\Phi(Z_U)\in\bigcup_j\mathbb B_j(\theta^2))$.
Here $\Phi$ is the preprocessing map from the chosen route in A6, or any fixed measurable preprocessing map for which the displayed probability is being evaluated.

\noindent\textbf{(iii) Under (U-Norm) + A6$'$ (for $\theta\in[0,1)$):}
\begin{equation}\label{eq:far-norm}
\FAR(\theta)\le C\exp\!\left(-\frac{(d-1)t(\theta^2)^2}{2}\right)+C\,\delta_{\mathrm{cap}}.
\end{equation}
If $\delta_{\mathrm{cap}}=0$ and $\Prob(Z_U\neq 0)=1$, the exact-isotropy lower bound
\[
\pcap(t(\theta^2/2);d)\le\FAR(\theta)
\]
also holds. Under the same exact-isotropy and nonzero-origin assumptions, and additionally A1$''$, the disjoint-cap simplex lower bound
\[
C\,\pcap(t((\theta/3)^2);d)\le\FAR(\theta)
\]
holds.

\noindent\textbf{(iv) Under (U-RadIso) + A6$'$ (for $\theta\in[0,1)$):}
\begin{equation}\label{eq:far-radiso}
\FAR(\theta)\le C\,\E\!\left[\exp\!\left(-\frac{d-1}{2}\min\{1,\tau_\theta(R_U)\}^2\right)\right],
\end{equation}
where $\tau_\theta(r):=\frac{t(\theta^2)}{2}\cdot\frac{r^2+R^2}{rR}$ for $r>0$, and we use the endpoint convention $\min\{1,\tau_\theta(0)\}:=1$ because $R_U=0$ is rejected for $\theta<1$.
\end{theorem}

\begin{proof}[Proof sketch]
(i) follows from Lemma~\ref{lem:quantile}, total probability, and clipping at $1$ when the raw threshold exceeds the score range.
(ii) uses Proposition~\ref{prop:global-comparison} and Theorem~\ref{thm:geometry-U2}(ii).
(iii): on the sphere, the upper bound comes from a union of spherical caps together with (U-Norm) and A6$'$. The exact-isotropy lower bound uses the inner inclusion from Theorem~\ref{thm:geometry-U2}(iv); the simplex lower bound additionally uses Lemma~\ref{lem:rms-am} and cap disjointness.
(iv): condition on $R_U$, use isotropy of $V_U$, and apply A6$'$ pointwise.
Full proofs in Appendix~\ref{app:far}.
\end{proof}

\textbf{Interpretation.}
Equation~\eqref{eq:far-norm} shows that, in the normalized-route analysis, the FAR upper bound for $U(\Phi(Z_U))$ is an exponential term plus an additive anisotropy floor $C\,\delta_{\mathrm{cap}}$.
Hence under exact isotropy ($\delta_{\mathrm{cap}}=0$), the normalized-route upper bound decays exponentially in $d$ for any fixed $\theta<1$; under the radial-isotropic model, equation~\eqref{eq:far-radiso} gives an exponential-in-$d$ FAR upper bound for $U(Z_U)$.
The term $t(\theta^2)$ quantifies how much of the sphere is \emph{not} covered by the acceptance caps: smaller $\theta$ yields larger $t$, hence faster decay of the upper bound.

\subsection{Why the FAR bound can be loose in practice}\label{sec:far-loose}

Although Theorem~\ref{thm:far-krr} provides exponential-in-$d$ FAR scaling (the key qualitative insight), several factors contribute to conservatism in the bound, particularly in the prefactors and additive terms.

The proof uses a union bound over $C$ spherical caps, yielding $\FAR(\theta)\le C\exp(-(d-1)t(\theta^2)^2/2)+C\,\delta_{\mathrm{cap}}$. The term $\delta_{\mathrm{cap}}$ captures deviation of the true unknown-feature angular distribution from the ideal uniform sphere; when unknowns concentrate along a few learned directions, this term can be substantial. The union bound itself is conservative: when cap events overlap (which they often do in practice), summing $C$ marginal probabilities can substantially overcount their union.

Beyond the union bound, the proof strategy itself introduces looseness. It assumes isotropy of unknown features in all directions $v\in S^{d-1}$, but learned embeddings can develop directions aligned with simplex geometry since features are trained with respect to those very prototypes. Additionally, the proof uses $\{U\le\theta\}\subseteq\{U_2\le\theta^2\}$ (Proposition~\ref{prop:global-comparison}) to shift the analysis from the operational score $U$ to the auxiliary score $U_2$, which enlarges the acceptance region and thus inflates the FAR bound. The spherical-cap approximation itself, namely $\pcap(t;d)\le\exp(-(d-1)t^2/2)$ from A6$'$, can be crude for moderate $d$ and $t$ close to $1$.

These constants should be interpreted as conservative robustness terms rather than tight numerical predictions. The main qualitative content is the exponential dependence on $d$ in the bound. The result is mainly a tool for comparing dimension trends, not for predicting absolute FAR values without checking the distributional assumptions.

%% ============================================================
%% 7. UCDSC LOSS ANALYSIS
%% ============================================================
\section{UCDSC Loss Analysis}\label{sec:loss-analysis}

We now analyze the three components of the UCDSC loss~\eqref{eq:Lintra}--\eqref{eq:Lu}.
Proposition~\ref{prop:intra} and Theorem~\ref{thm:combined} yield deterministic in-sample score bounds for the realized training known and auxiliary background features, while Proposition~\ref{prop:grad} identifies a local gradient mechanism for the $L_u$ term.

\begin{proposition}[Small $L_{\mathrm{intra}}$ implies a small average training $U$-score]\label{prop:intra}
Under A1 (distinct prototypes): if $L_{\mathrm{intra}}\le\delta^2$ for some $\delta\ge 0$, then $\frac{1}{n}\sum_i U(f_i)\le 2\delta/\Delta$, where $\Delta=\min_{i\neq j}|s_i-s_j|$.
Under A1$''$ (simplex), this specializes to $2\delta/D$.
\end{proposition}

\begin{proof}
Since $d_{\min}(f_i)\le\|f_i-s_{y_i}\|$, the QM--AM inequality gives $\frac{1}{n}\sum_i d_{\min}(f_i)\le\delta$.
By Theorem~\ref{thm:generic-ball}, $\mu(f_i)\ge \Delta/2$, so $U(f_i)\le 2d_{\min}(f_i)/\Delta$.
Averaging gives the result. Under A1$''$, $\Delta=D$ and Lemma~\ref{lem:lipschitz}(i) gives $\mu(f_i)\ge D/2$.
\end{proof}

\begin{proposition}[Zero $L_o$ lower-bounds auxiliary background scores]\label{prop:outlier}
Under A1 (distinct prototypes), assume every class $j\in[C]$ appears among $\{y_i\}_{i=1}^n$.
If $L_o=0$, then $\|f_k^{\mathrm{bg}}-s_j\|\ge\sqrt{m}$ for all classes $j$ and all background samples $k$.
Consequently $U(f_k^{\mathrm{bg}})\ge\sqrt{m}/(D_{\max}+\sqrt{m})$ for all background samples $f_k^{\mathrm{bg}}$, where $D_{\max}:=\max_{i\neq j}|s_i-s_j|$.
Under A1$''$ (simplex), $D_{\max}=D$ and we recover $U(f_k^{\mathrm{bg}})\ge\sqrt{m}/(D+\sqrt{m})$.
\end{proposition}

\begin{proof}
All hinge summands in $L_o$ are nonnegative.
If $L_o=0$, each hinge is zero, so for every training index $i$ and background index $k$,
\[
\|f_k^{\mathrm{bg}}-s_{y_i}\|^2\ge m+\|f_i-s_{y_i}\|^2\ge m.
\]
Because every class appears among the labels $y_i$, this gives $\|f_k^{\mathrm{bg}}-s_j\|\ge\sqrt{m}$ for every $j,k$, and hence $d_{\min}(f_k^{\mathrm{bg}})\ge\sqrt{m}$.
The triangle inequality gives $\mu(f_k^{\mathrm{bg}})\le d_{\min}(f_k^{\mathrm{bg}})+D_{\max}$, yielding the bound.
\end{proof}

\begin{proposition}[Gradient structure of $L_u$]\label{prop:grad}
For $f_i\neq s_j$ for all $j$ and assuming the nearest prototype $j^*=j^*(f_i)$ is unique:
\begin{equation}\label{eq:grad}
\begin{aligned}
\nabla_{\!f_i} U(f_i)
&=\frac{1}{\mu}\Bigg(
\frac{f_i-s_{j^*}}{\|f_i-s_{j^*}\|}\\
&\hspace{2.6em}
-\frac{U(f_i)}{C-1}
\sum_{k\neq j^*}\frac{f_i-s_k}{\|f_i-s_k\|}
\Bigg).
\end{aligned}
\end{equation}
Here $\mu=\mu(f_i)$, and $\nabla_{f_i}L_u=(1/n)\nabla_{f_i}U(f_i)$ for the averaged loss.
Under gradient descent on the pointwise term $U(f_i)$, write
\[
\begin{aligned}
-\nabla U(f_i)&=-\frac{1}{\mu}u_{j^*}+\frac{U(f_i)}{\mu}\,q,\\
u_j&:=\frac{f_i-s_j}{\|f_i-s_j\|},\\
q&:=\frac{1}{C-1}\sum_{k\neq j^*}u_k .
\end{aligned}
\]
The first component moves $f_i$ toward $s_{j^*}$ and decreases $d_{\min}$.
The second component is proportional to $\nabla\mu(f_i)$ and, applied alone, increases the average non-nearest distance $\mu$.
The sum is the negative gradient of $U$, hence gives first-order descent for $U$; the two displayed components should not be interpreted as separately decreasing $U$.
\end{proposition}

\begin{theorem}[Conditional empirical separation for training known and auxiliary background samples]\label{thm:combined}
Under A1--A2, A4, assume every class $j\in[C]$ appears among $\{y_i\}_{i=1}^n$.
Let $D_{\max}:=\max_{i\neq j}|s_i-s_j|$.
If the realized embeddings satisfy $L_o=0$, then:
\begin{enumerate}[label=(\roman*),leftmargin=1.8em,itemsep=1pt]
\item Each training sample satisfies $U(f_i)\le 2\|f_i-s_{y_i}\|/\Delta$.
\item Each auxiliary background sample satisfies $U(f_k^{\mathrm{bg}})\ge\sqrt{m}/(D_{\max}+\sqrt{m})$.
In particular, if $\sqrt{m}\ge D_{\max}$, then $U(f_k^{\mathrm{bg}})\ge 1/2$.
\item A separating threshold $\theta^*\in[0,1]$ satisfying
\[
\max_i U(f_i)<\theta^*<\min_k U(f_k^{\mathrm{bg}})
\]
exists whenever the following sufficient condition holds:
\begin{equation}\label{eq:sep-condition}
\max_i\|f_i-s_{y_i}\|<\frac{\Delta\sqrt{m}}{2(D_{\max}+\sqrt{m})}.
\end{equation}
\end{enumerate}
Under A1$''$ (simplex), $\Delta=D_{\max}=D$ and these reduce to the simplex-specific bounds.
\end{theorem}

\begin{proof}
(i): $d_{\min}(f_i)\le\|f_i-s_{y_i}\|$ and $\mu(f_i)\ge \Delta/2$ (Theorem~\ref{thm:generic-ball}).
(ii): Proposition~\ref{prop:outlier}.
(iii): The known upper bound $2\max_i\|f_i-s_{y_i}\|/\Delta$ must be strictly less than the background lower bound; \eqref{eq:sep-condition} is a sufficient condition for this strict separation.
\end{proof}

This is an in-sample statement about the realized training embeddings and the realized auxiliary background embeddings.
It does not by itself imply rejection guarantees for future unknown classes, nor does it directly quantify open-space risk under a test-time unknown distribution.
Within this scope, Theorem~\ref{thm:combined} shows how prototype closeness of known features and margin-based separation from the auxiliary background set induce score separation.
Proposition~\ref{prop:intra} provides only an average training-score bound and does not by itself imply the uniform condition in~\eqref{eq:sep-condition}.
The role of $L_u$ is different: Proposition~\ref{prop:grad} and Corollary~\ref{cor:lu} provide a local negative-gradient mechanism describing how minimizing $L_u$ can combine nearest-prototype attraction with an increase in average non-nearest distance.

%% ============================================================
%% 8. PRACTICAL COROLLARIES
%% ============================================================
\section{Practical Corollaries}\label{sec:corollaries}

\begin{corollary}[Threshold selection]\label{cor:threshold}
Under A1$'$--A2 and A5, for $\varepsilon\in(0,1)$, setting
\begin{equation}\label{eq:theta-eps}
\begin{aligned}
\theta_\varepsilon
&=\min\{1,\Theta_\varepsilon\},\\
\Theta_\varepsilon
&:=\max_y\left(\E[U(Z)\mid Y\!=\!y]
+\tau_y\sqrt{2\log(1/\varepsilon)}\right)\\
&=\max_y\left(\E[U(Z)\mid Y\!=\!y]
+L_U\sigma_y\sqrt{2\log(1/\varepsilon)}\right).
\end{aligned}
\end{equation}
This guarantees population known-class acceptance $1-\KRR(\theta_\varepsilon)=\Prob(U(Z)\le\theta_\varepsilon)\ge 1-\varepsilon$.
In the simplex specialization A1$''$, one may take $L_U=4/D$.
\end{corollary}

\noindent\textit{Validation use.}
In practice, estimate $\E[U(Z)\mid Y\!=\!y]$ and the score-tail scales $\tau_y$ (or equivalently the rescaled score-tail parameters $\sigma_y=\tau_y/L_U$) from a validation set.
This plug-in step is operational advice rather than a finite-sample theorem unless a safety margin or a correction such as Remark~\ref{rem:finite-sample} is added.

\begin{corollary}[Sufficient dimension for target FAR]\label{cor:dim}
Under the hypotheses of Theorem~\ref{thm:far-krr}(iii), with (U-Norm), $\delta_{\mathrm{cap}}=0$, $\theta\in(0,1)$, and $\eta\in(0,1)$, it suffices for $\FAR(\theta)\le\eta$ that
\begin{equation}\label{eq:min-d}
d\ge 1+\frac{2\log(C/\eta)}{t(\theta^2)^2}.
\end{equation}
Here $\log$ is natural, and the integer dimension should be rounded up.
For $C=8$, $\theta=0.5$, $\eta=0.01$: $1+2\log(800)/t(0.25)^2\approx 26.50$, hence integer $d\ge 27$.
This is a modest dimension relative to common learned embedding sizes.
\end{corollary}

\begin{corollary}[Expand factor invariance under normalization]\label{cor:expand}
For fixed $\theta,d,C$, and $\delta_{\mathrm{cap}}$, the normalized-route FAR upper bound~\eqref{eq:far-norm} contains no explicit dependence on the prototype radius $R$.
If $\theta$ is instead selected by the known-recall rule, changing $R$ can affect the selected threshold through learned score means $m_y$, score-tail scales $\tau_y=L_U\sigma_y$, and learning-induced changes in the unknown angular distribution, including $\delta_{\mathrm{cap}}$.
Without normalization (U-RadIso), the displayed bound depends on $R$ through $\tau_\theta(R_U)$; whether increasing $R$ tightens the bound depends on how the unknown-radius distribution of $R_U$ is positioned relative to $R$.
\end{corollary}

\begin{corollary}[Local $L_u$ descent mechanism]\label{cor:lu}
Assume the differentiability conditions of Proposition~\ref{prop:grad} hold at a training feature $f_i$.
The negative gradient of the pointwise term $U(f_i)$ combines (a) attraction toward the nearest prototype $s_{j^*(f_i)}$ and (b) a component that, applied alone, increases the average non-nearest distance $\mu(f_i)$.
When $j^*(f_i)=y_i$, the attraction component is aligned with the prototype pull encouraged by $L_{\mathrm{intra}}$.
Thus $L_u$ provides a local first-order mechanism distinct from $L_{\mathrm{intra}}$: it can combine nearest-prototype attraction with increased average non-nearest distance.
This mechanism is consistent with, but does not prove, empirical OSCR improvements from adding $L_u$ without background data \cite{aditya2025ucdsc,dhamija2018reducing}; it does not by itself imply global open-space reduction or test-time AUROC/OSCR gains.
\end{corollary}

\subsection{Practical Estimation of $\sigma$}\label{sec:sigma-est}

Assumption~\ref{ass:conc} introduces the per-class upper-tail scales $\tau_y$ directly.
The threshold rule (Theorem~\ref{thm:far-krr}(i)) can be written in terms of $\tau_y$ or, equivalently, the rescaled score-tail parameters $\sigma_y=\tau_y/L_U$.
We now give a concrete diagnostic procedure to estimate these quantities from a labeled validation set.

\noindent\textbf{Validation-based $\sigma$ proxy.}
Let $\{(x_i,y_i)\}_{i=1}^N$ be a held-out validation set.
Compute $z_i=f(x_i)$ and $u_i:=U(z_i)$.
For each class $y\in[C]$, let $I_y:=\{i:y_i=y\}$ and assume $|I_y|>0$.
Let $\hat F_y(u):=|I_y|^{-1}\sum_{i\in I_y}\mathbf 1\{u_i\le u\}$ and use the empirical inverse-CDF convention
\[
\hat q_{p,y}:=\inf\{u:\hat F_y(u)\ge p\}.
\]
Define:
\begin{enumerate}[label=(\alph*),leftmargin=1.8em,itemsep=1pt]
\item \textbf{Empirical mean:}\; $\hat m_y:=|I_y|^{-1}\sum_{i\in I_y}u_i$.
\item \textbf{Empirical quantiles:}\; $\hat q_{1-\varepsilon,y}:=$ the $(1-\varepsilon)$-quantile of $\{u_i\}_{i\in I_y}$.
\item \textbf{Upper-tail proxy:}\; For a grid $\mathcal E\subset(0,1)$ (e.g.\ $\mathcal E=\{0.1,0.05,0.01\}$),
\begin{equation}\label{eq:tau-hat}
\hat\tau_y:=\max_{\varepsilon\in\mathcal E}\;\frac{\bigl[\hat q_{1-\varepsilon,y}-\hat m_y\bigr]_+}{\sqrt{2\log(1/\varepsilon)}},
\end{equation}
with $[\cdot]_+$ as defined in Section~\ref{sec:prelim}.
\end{enumerate}
Then back out $\hat\sigma_y=\hat\tau_y/L_U$, and set $\hat\sigma:=\max_{y\in[C]}\hat\sigma_y$. In the simplex specialization A1$''$, this simplifies to $\hat\sigma_y=(D/4)\hat\tau_y$.

\textbf{Justification.}
By Lemma~\ref{lem:quantile}, the population quantities satisfy
\[
\frac{[q_{1-\varepsilon}(U(Z)\mid Y\!=\!y)-\E[U(Z)\mid Y\!=\!y]]_+}{\sqrt{2\log(1/\varepsilon)}}\le\tau_y
\]
for every $\varepsilon\in(0,1)$.
Replacing the population mean and quantiles by their empirical counterparts over a grid $\mathcal E$ yields a practical proxy for $\tau_y$.
Even at the population level this inversion is only a lower-envelope diagnostic for the tail behavior over the chosen grid $\mathcal E$, not a characterization of the full A5 scale.
With empirical inputs it can substantially underestimate $\tau_y$, especially for rare tails or small validation sets; consequently, the resulting threshold $\hat\theta_\varepsilon$ can be optimistic (too small).
In practice, adding a safety margin or applying a finite-sample correction such as Remark~\ref{rem:finite-sample} can reduce this risk, but the proxy should not be read as a guarantee that $\hat\sigma\ge\sigma$.
The clipping $[\cdot]_+$ handles the (rare) case where the empirical quantile falls below the empirical mean due to finite-sample fluctuation.

\textbf{Plugging in.}
With $\hat\sigma$ in hand, the threshold selection proxy (Corollary~\ref{cor:threshold}) becomes fully data-driven:
\begin{equation}\label{eq:theta-practical}
\begin{aligned}
\hat\theta_\varepsilon
&:=\min\!\left\{1,\max_{y\in[C]}\left(\hat m_y+\hat\tau_y\sqrt{2\log(1/\varepsilon)}\right)\right\}\\
&=\min\!\left\{1,\max_{y\in[C]}\left(\hat m_y+L_U\hat\sigma_y\sqrt{2\log(1/\varepsilon)}\right)\right\},
\end{aligned}
\end{equation}
which provides a practical, theorem-motivated alternative to grid-search threshold tuning.

\begin{remark}[Finite-sample correction]\label{rem:finite-sample}
For small validation sets ($|I_y|\lesssim 50$), the empirical quantiles $\hat q_{1-\varepsilon,y}$ are noisy.
A Dvoretzky--Kiefer--Wolfowitz (DKW) correction can be used with the inverse-CDF convention above:
replace $\hat q_{1-\varepsilon,y}$ by $\hat q_{1-\varepsilon+\delta_n,y}$, where
\[
\delta_n=\sqrt{\frac{\log(2/\alpha)}{2|I_y|}}
\]
for classwise confidence level $1-\alpha$.
This corrected quantile level is feasible only when $\delta_n\le\varepsilon$; otherwise the distribution-free fallback from $U\in[0,1]$ is the trivial upper bound $1$.
For simultaneous confidence over all classes, apply a union bound, for example by replacing $\alpha$ with $\alpha/C$.
If forming an upper-confidence finite-grid tail proxy, one should also add a mean margin such as $h_y=\sqrt{\log(2/\alpha)/(2|I_y|)}$ in the numerator, e.g. $[\hat q-\hat m_y+h_y]_+$.
Even with these corrections, the result controls only the chosen finite-grid quantile proxy, not the full A5 sub-Gaussian scale, and therefore does not prove $\hat\sigma\ge\sigma$.
\end{remark}

\begin{remark}[Connection to expand factor]\label{rem:sigma-R}
Under the generic A1/A1$'$ Lipschitz bound $L_U=4/\Delta$,
\[
\hat\sigma_y=\frac{\hat\tau_y}{L_U}=\frac{\Delta}{4}\,\hat\tau_y.
\]
Hence the dependence on the expand factor $R$ enters through how the minimum separation $\Delta$ of the chosen prototype family scales with $R$. If the prototypes are obtained by scaling a fixed angular code, then $\Delta$ is linear in $R$. In the simplex specialization, $\Delta=D=R\sqrt{2C/(C-1)}$, recovering $\hat\sigma_y=(D/4)\hat\tau_y$.
In practice, $\hat\tau_y$ (the directly observable tail scale of $U$-scores) is the more natural quantity.
Under A5 alone, $\sigma_y$ should be viewed as a rescaled score-tail parameter for $U$, not as a generic feature-space concentration parameter.
\end{remark}

The complete estimation pipeline is summarized in Algorithm~\ref{alg:sigma}.

\begin{algorithm}[t]
\caption{Practical estimation of $\hat\sigma$ and $\hat\theta_\varepsilon$}\label{alg:sigma}
\begin{algorithmic}[1]
\Require Trained model $f$, prototypes $\{s_1,\dots,s_C\}$, validation set $\{(x_i,y_i)\}_{i=1}^N$ with each class represented, target known-recall miss rate $\varepsilon$, quantile grid $\mathcal E$
\Ensure Threshold proxy $\hat\theta_\varepsilon$, score-tail proxy $\hat\sigma$
\State Compute $z_i\gets f(x_i)$ and $u_i\gets U(z_i)$ for all $i$
\State Compute $\Delta\gets \min_{i\neq j}\|s_i-s_j\|$, \; $L_U\gets 4/\Delta$ \Comment{if simplex: $L_U=4/D$}
\For{each class $y\in[C]$}
    \State $I_y\gets\{i:y_i=y\}$
    \State $\hat m_y\gets|I_y|^{-1}\sum_{i\in I_y}u_i$
    \For{each $\varepsilon'\in\mathcal E$}
        \State Compute $\hat q_{1-\varepsilon',y}$ from $\{u_i\}_{i\in I_y}$
    \EndFor
    \State $\hat\tau_y\gets\max_{\varepsilon'\in\mathcal E}\;[\hat q_{1-\varepsilon',y}-\hat m_y]_+/\sqrt{2\log(1/\varepsilon')}$
    \State $\hat\sigma_y\gets\hat\tau_y/L_U$
\EndFor
\State $\hat\sigma\gets\max_y\hat\sigma_y$
\State $\hat\theta_\varepsilon\gets\min\!\left\{1,\max_y\left(\hat m_y+\hat\tau_y\sqrt{2\log(1/\varepsilon)}\right)\right\}$
\State \Return $\hat\theta_\varepsilon,\;\hat\sigma$
\end{algorithmic}
\end{algorithm}

\subsection{Practical Estimation of $\delta_{\mathrm{cap}}$}\label{sec:delta-cap-est}

To apply the FAR bound (Theorem~\ref{thm:far-krr}(iii)), one must estimate $\delta_{\mathrm{cap}}$ from data.
Collect a set of unknown embeddings $\{z_i\}_{i=1}^N$ and radially normalize each to radius $R$ via $\Phi(z_i)=Rz_i/\|z_i\|$ (discard $z_i=0$), so $\Phi(z_i)/R\in S^{d-1}$.
Draw $M$ random directions $v_m\sim\mathrm{Unif}(S^{d-1})$, and choose thresholds $t$ on a fine grid $\mathcal T\subset[0,1]$.
For each sampled direction and threshold, compute the one-direction empirical cap tail
\[
\hat p_m(t)=\frac{1}{N}\sum_{i=1}^N\mathbf{1}\{\langle\Phi(z_i)/R,v_m\rangle\ge t\}.
\]
Compute the corresponding spherical-cap probability $\pcap(t;d)=\Prob(\langle W,v\rangle\ge t)$ for $W\sim\mathrm{Unif}(S^{d-1})$.
Then define the sampled-direction diagnostic
\[
\widehat\delta_{\mathrm{cap}}=\max_{m\in[M],\,t\in\mathcal T}|\hat p_m(t)-\pcap(t;d)|.
\]
Optionally, bootstrap over embeddings and directions to quantify finite-sample uncertainty.

\noindent\textit{Caveat.}
Assumption~\ref{ass:unknown} (U-Norm) defines $\delta_{\mathrm{cap}}$ as a worst-case (supremum-over-$v$) deviation.
The finite estimator above scans a random set of directions and a threshold grid, so it is a Monte Carlo lower bound on the true supremum rather than a certified upper bound.
For isotropic unknowns, all directions are distributionally equivalent, but the maximum over sampled directions and thresholds still includes finite-sample fluctuation.
For anisotropic unknowns, the same diagnostic is intended to reveal substantially larger cap deviations; a rigorous upper bound would require an $\varepsilon$-net or covering argument over $S^{d-1}$ combined with a DKW-type bound over thresholds.

\subsection{Numerical Illustrations}\label{sec:numerical}

We report results from synthetic validation experiments implementing Algorithm~\ref{alg:sigma} and the $\delta_{\mathrm{cap}}$ estimator.

\textbf{$\hat\sigma$ estimation.}
For a regular-simplex illustration with $C=4$ classes, $R=50$, $n=30$ samples per class, and $\varepsilon=0.01$, Algorithm~\ref{alg:sigma} yields per-class results in Table~\ref{tab:sigma}.
Here $D=81.6$ and $L_U=4/D$, so the reported $\hat\sigma_y$ values use the simplex specialization.
The deterministic synthetic script reports $\hat\theta_{0.01}=0.092$, well below $1$, confirming that known features cluster tightly around prototypes.
For this deliberately small $n=30$ example, the DKW finite-sample correction at classwise $\alpha=0.05$ widens the threshold to the trivial upper bound $1$, illustrating why the uncorrected estimate should be read as a diagnostic proxy rather than a finite-sample guarantee.

\begin{table}[t]
\centering
\caption{Per-class $\hat\sigma$ estimation for the regular-simplex illustration ($C=4$, $R=50$, $D=81.6$, $L_U=4/D$).
Values of $\hat\sigma_y$ are computed from unrounded $\hat\tau_y$; the displayed $\hat\tau_y$ are rounded to four decimal places.}\label{tab:sigma}
\small
\input{table1_sigma_proxy.tex}
\end{table}

\textbf{$\widehat\delta_{\mathrm{cap}}$ estimation.}
With $N=1{,}500$ isotropic unknowns in $d=16$, $R=50$, and $M=4{,}096$ random directions, the generated diagnostic gives $\widehat\delta_{\mathrm{cap}}=4.36\times 10^{-2}$.
This is consistent with the finite-sample supremum error expected when scanning many directions and thresholds; for anisotropic unknown distributions, the same estimator is intended to reveal substantially larger cap deviations.

\textbf{FAR vs.\ dimension $d$ (synthetic sanity check).}
To illustrate Theorem~\ref{thm:far-krr}(iii), we sample $5{,}000$ unknowns uniformly on $S^{d-1}$ (scaled to $R=50$) for $C=4$, $\theta=0.2$, and $d\in\{3,5,8,16,32,64\}$.
Table~\ref{tab:far-vs-d} and Figure~\ref{fig:far-vs-d} show that empirical FAR drops rapidly; for $d\ge 8$, we observe zero false acceptances among $5{,}000$ sampled unknowns, while the theoretical bound captures the expected exponential trend.

\begin{table}[t]
\centering
\caption{FAR vs.\ dimension $d$. Empirical FAR from $5{,}000$ uniform-sphere unknowns and the Theorem~\ref{thm:far-krr}(iii) bound ($C=4$, $R=50$, $\theta=0.2$, $t(\theta^2)\approx 0.947$). The bound column is clipped at $1$, i.e.\ it reports $\min\{1,C\exp(-(d-1)t(\theta^2)^2/2)\}$. Entries reported as $<\!2\times 10^{-4}$ correspond to zero observed false acceptances among $5{,}000$ samples.}\label{tab:far-vs-d}
\small
\input{table2_far_vs_d.tex}
\end{table}

\begin{figure}[t]
\centering
\includegraphics[width=\columnwidth]{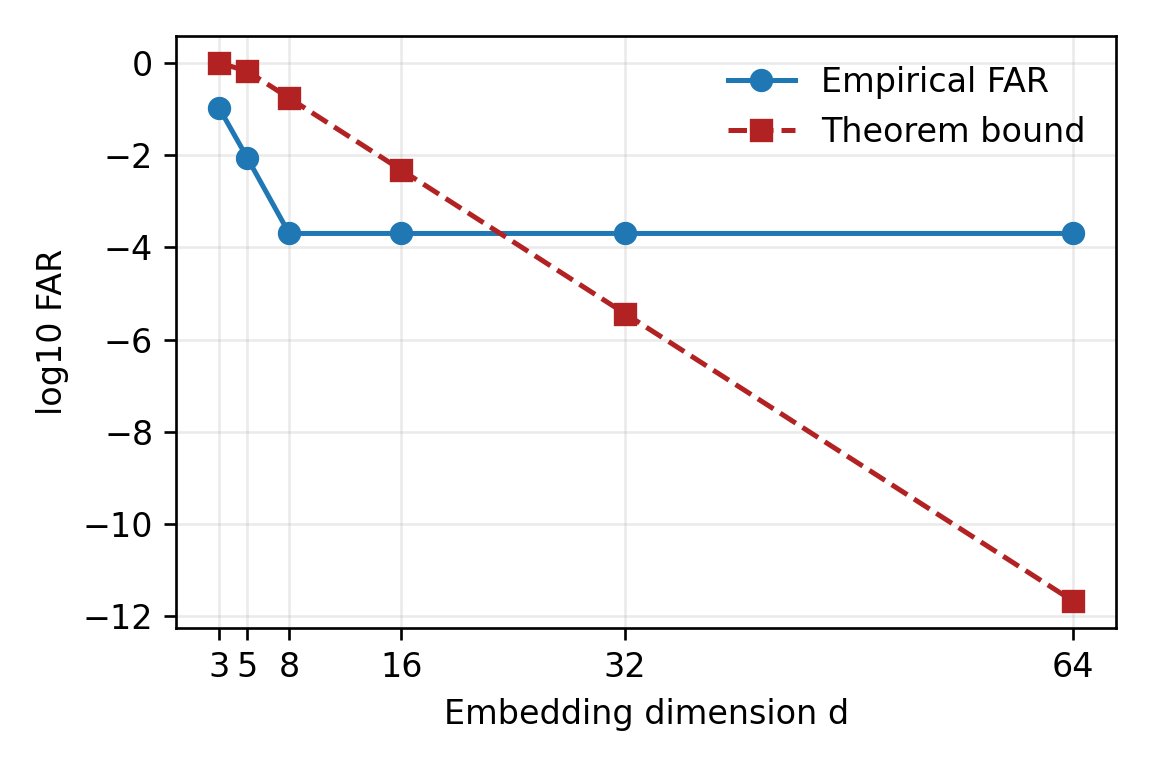}
\caption{$\log_{10}(\mathrm{FAR})$ vs.\ embedding dimension $d$ for uniform-sphere unknowns ($C=4$, $R=50$, $\theta=0.2$). Blue: empirical FAR, with zero-count cases displayed at the detection floor $1/5000=2\times 10^{-4}$; red dashed: Theorem~\ref{thm:far-krr}(iii) upper bound. The linear trend in log-scale is consistent with exponential-in-$d$ FAR decay in this isotropic example.}
\label{fig:far-vs-d}
\end{figure}

\textbf{Prospective operational note: data-driven thresholding via $\hat\sigma$.}
Algorithm~\ref{alg:sigma} provides a theorem-motivated threshold-selection rule for future implementations: given a target known-class miss rate $\varepsilon$, it selects $\hat\theta_\varepsilon$ from validation-set $U$-score quantiles via equation~\eqref{eq:theta-practical}.
This is intended as operational guidance rather than validation.
Without a safety margin or DKW correction, the raw plug-in threshold can be optimistic because $\hat\tau_y$ may underestimate the true tail scale.
The rule controls known-class recall, not AUROC or OSCR directly.
Datasets with larger intra-class spread would be expected to yield larger $\hat\tau_y$, if that spread is actually observed on validation data.

%% ============================================================
%% 9. SCORER-AGNOSTIC VALIDATION
%% ============================================================
%%% Section 9: Experimental Validation of Theoretical Claims
%%% Restructured around three theory-driven validation questions.
%%% Drop-in replacement for the current Section 9.

\section{Experimental Validation}
\label{sec:experiments}

The preceding sections establish analytic properties of balanced prototype geometry: Lipschitz stability of the ratio scores (Theorem~\ref{thm:generic-ball}), acceptance-ball characterisation of $U_2$ (Theorem~\ref{thm:geometry-U2}), the sharp dichotomy between exact simplex equidistance and below-simplex defect (Theorem~\ref{thm:dichotomy}), harmonic constructions that operate below the simplex dimension constraint (Section~5.2), the tangent-ball separation diagnostic $\tau_{\mathrm{sep}}$ (Corollary~\ref{cor:tau-sep}), and FAR bounds under isotropic unknown-feature assumptions (Theorem~\ref{thm:far-krr}).
These are \emph{structural} claims about how the scoring geometry behaves; they do not predict that the ratio scores $U$ and $U_2$ will dominate every post-hoc OOD detector on every benchmark.
The experiments in this section are therefore designed to evaluate the theory-facing claims, not to claim state-of-the-art detection performance.
We organise the evaluation around three questions.

\paragraph{Q1.}  Do trained prototype representations actually exhibit the geometry diagnostics suggested by the theory---centered barycentre, equal norms, and controlled pairwise-distance variation?

\paragraph{Q2.}  Does the harmonic construction (the $C$-gon exposition and its multi-frequency extension) produce usable representations when the embedding dimension falls below $C{-}1$, where simplex prototypes are infeasible?

\paragraph{Q3.}  How do the theory's ratio scores $U$, $U_2$ compare to post-hoc logit-based scorers on \emph{the same} trained embeddings, and does this comparison align with the theory's scope?

\subsection{Protocol}

\paragraph{Datasets.}
We evaluate on five open-set benchmarks spanning natural images and medical imaging:
CIFAR+10 and CIFAR+50 use CAC splits with 4 known CIFAR-10 classes and 10 or 50 unknown CIFAR-100 classes, respectively;
CIFAR-100 uses an 80 known / 20 unknown holdout;
BloodMNIST has 8 cell-type classes split 4 known / 4 unknown;
and PathMNIST has 9 colon pathology tissue classes split 6 known / 3 unknown.
Each dataset uses 5 disjoint train/val/test splits.

\paragraph{Representations.}
Four representation methods share a common ResNet-18 backbone and embedding dimension $d{=}8$:
\textbf{(i)}~\emph{CE}:\ standard cross-entropy with a linear classifier head;
\textbf{(ii)}~\emph{Simplex}:\ equiangular tight frame prototypes (requires $d \ge C{-}1$);
\textbf{(iii)}~\emph{Harmonic}:\ the multi-frequency harmonic construction from Section~5.2 (defined for all $d \ge 2$);
\textbf{(iv)}~\emph{Learned}:\ \texttt{nn.Parameter} centres initialised from harmonic codes and updated end-to-end.
The scorer-agnostic prototype experiments use the implementation loss actually evaluated in code: cross-entropy on negative squared prototype distances, plus a compactness penalty and a true-label squared-ratio regulariser.
The full Section~7 objective with an explicit open/background term is the theoretical UCDSC objective and is not invoked in these no-outlier-exposure prototype runs.

\paragraph{Scorer-agnostic evaluation.}
After training each representation once, we freeze the embeddings and evaluate with 11~scorers: $U$ (unsquared), $U_2$ (squared), min-distance, KNN\textsubscript{50}, MSP, MaxLogit, Energy, ODIN, ViM, ReAct, and OpenMax.
For prototype representations, ratio and distance scores use the trained prototype centres.
For CE representations, these diagnostics use train-set class-mean centres.
Logit-based scorers use the native CE head for CE representations and a frozen common linear head (trained on known-class embeddings only) for prototype representations.
This isolates representation quality from scorer choice.

\paragraph{Calibration.}
All thresholds are calibrated on validation known/unknown scores only; test data is never used for threshold selection.
Metrics are AUROC and OSCR, computed on held-out test partitions.
All stored detection scores are OOD-oriented (higher means more unknown); OSCR is therefore integrated using ID confidence $-\mathrm{score}$.
For OSCR curves we report the upper monotone envelope of the empirical CCR--FPR staircase before trapezoidal integration, matching the monotone operating-characteristic convention; AUROC and validation-calibrated threshold metrics are unaffected.

%% =====================================================================
\subsection{Q1: Does Balanced Geometry Emerge in Practice?}
\label{sec:q1}

Table~\ref{tab:geometry} reports representative prototype geometry diagnostics averaged over splits at $d{=}8$.
Three properties predicted by the theory can be directly measured.

\paragraph{Centring.}
The barycentre norm $\|\bar{\mathbf{s}}\|$ should be near zero for balanced codes (A1$'$ and Theorem~\ref{thm:geometry-U2} assume centred geometry).
Simplex and harmonic representations achieve $\|\bar{\mathbf{s}}\| = 0.000$ by construction.
Learned prototypes stay close ($0.01$--$0.11$), while CE class-mean centres are heavily off-centre ($1.8$--$9.4$), confirming that standard cross-entropy training does not produce centred representations.

\paragraph{Equal norms.}
The coefficient of variation of prototype radii ($\mathrm{CV}_r$) measures deviation from equal-norm placement.
Simplex and harmonic codes achieve $\mathrm{CV}_r = 0.000$; learned prototypes reach $\mathrm{CV}_r \le 0.065$; CE centres show $\mathrm{CV}_r = 0.10$--$0.26$.
The theory's ball geometry (Theorem~\ref{thm:geometry-U2}) relies on equal-norm centres together with zero barycentre, and this is satisfied by construction for fixed codes and approximately maintained under end-to-end learning.

\paragraph{Equidistant spacing.}
The coefficient of variation of pairwise distances ($\mathrm{CV}_d$) captures deviation from one-distance/equiangular spacing.
Simplex prototypes achieve $\mathrm{CV}_d = 0.000$ (perfect ETF).
Harmonic codes reach $\mathrm{CV}_d = 0.04$--$0.06$; the regular $C$-gon exposition is balanced and equal-norm, but for $C>3$ its chord distances vary by angular separation, and the multi-frequency implementation improves but does not enforce one-distance equidistance.
CE centres are much more variable ($\mathrm{CV}_d = 0.09$--$0.23$).
CIFAR-100 is an informative stress test: with 80 classes packed into $d{=}8$, \emph{all} methods show elevated $\mathrm{CV}_d$ ($0.20$--$0.23$), confirming that the crowded regime degrades one-distance regularity regardless of the training objective.

\paragraph{Separation parameter.}
The diagnostic threshold $\tau_{\mathrm{sep}}$ from Corollary~\ref{cor:tau-sep} quantifies the tangent-ball margin available before the $U_2$ acceptance balls first touch.
Prototype methods consistently achieve higher $\tau_{\mathrm{sep}}$ ($0.23$--$0.35$) than CE ($0.01$--$0.24$).
CIFAR-100 produces $\tau_{\mathrm{sep}} \approx 0$ for all methods, correctly reflecting that $d \ll C{-}1$ leaves no geometric margin.

\begin{table}[t]
\centering
\caption{Representative prototype geometry diagnostics ($d{=}8$, averaged over splits).
$\|\bar{\mathbf{s}}\|$: barycentre norm.
$\mathrm{CV}_r$: radius coefficient of variation.
$\mathrm{CV}_d$: pairwise distance coefficient of variation.
$\tau_{\mathrm{sep}}$: tangent-ball separation diagnostic (Corollary~\ref{cor:tau-sep}).
}
\label{tab:geometry}
\scriptsize
\setlength{\tabcolsep}{2.5pt}
\begin{tabular}{@{}llrrrr@{}}
\toprule
Dataset & Rep. & $\|\bar{\mathbf{s}}\|$ & $\mathrm{CV}_r$ & $\mathrm{CV}_d$ & $\tau_{\mathrm{sep}}$ \\
\midrule
\multirow{4}{*}{CIFAR+10}
 & CE        & 1.839 & 0.105 & 0.093 & 0.241 \\
 & Simplex   & 0.000 & 0.000 & 0.000 & 0.354 \\
 & Harmonic  & 0.000 & 0.000 & 0.044 & 0.322 \\
 & Learned   & 0.021 & 0.017 & 0.019 & 0.328 \\
\midrule
\multirow{4}{*}{BloodMNIST}
 & CE        & 3.971 & 0.187 & 0.134 & 0.200 \\
 & Simplex   & 0.000 & 0.000 & 0.000 & 0.354 \\
 & Harmonic  & 0.000 & 0.000 & 0.044 & 0.322 \\
 & Learned   & 0.051 & 0.055 & 0.057 & 0.293 \\
\midrule
\multirow{3}{*}{CIFAR-100}
 & CE        & 5.729 & 0.153 & 0.227 & 0.011 \\
 & Harmonic  & 0.000 & 0.000 & 0.234 & 0.006 \\
 & Learned   & 0.041 & 0.130 & 0.200 & 0.020 \\
\bottomrule
\end{tabular}
\end{table}

\paragraph{Summary for Q1.}
The centred-barycentre and equal-norm properties assumed by the exact $U_2$ ball theory are realised by construction for simplex and harmonic codes, and approximately maintained by learned prototypes.
Small $\mathrm{CV}_d$ further diagnoses one-distance/simplex regularity, but it is not required by Theorem~\ref{thm:geometry-U2}.
CE representations are substantially worse on these geometry diagnostics.
The CIFAR-100 stress test confirms that exact one-distance simplex geometry cannot be maintained when $d \ll C{-}1$, even though balanced equal-norm harmonic codes remain feasible.

%% =====================================================================
\subsection{Q2: Does the $C$-gon Construction Work Below the Simplex Barrier?}
\label{sec:q2}

The harmonic construction (Section~5.2) uses the regular $C$-gon as its simplest 2D exposition and extends it via a multi-frequency code in higher dimensions.
It produces balanced equal-norm codes for any $d \ge 2$ regardless of the number of classes $C$.
This is its primary theoretical advantage over exact simplex ETF geometry, which requires $d \ge C{-}1$.

\paragraph{Feasibility.}
On CIFAR-100 ($C{=}80$, $d{=}8$), simplex prototypes are infeasible because $d < C{-}1 = 79$.
Harmonic prototypes remain defined, but in this extreme bottleneck they underperform CE in closed-set accuracy (harmonic: 35.4\% vs.\ CE: 61.0\% at $d{=}8$).
Learned prototypes initialized from harmonic codes recover CE-level closed-set accuracy (61.1\%), showing that feasibility below the simplex barrier is not the same thing as guaranteed discriminative performance.

\paragraph{Detection quality.}
Table~\ref{tab:q2_below_simplex} shows the best scorer for each representation on CIFAR-100.
Learned prototypes with KNN\textsubscript{50} reach 70.43 AUROC, within 1.72 points of CE+ODIN (72.15), despite the simplex ETF being unavailable.
Harmonic prototypes are weaker (64.96 best), reflecting that the multi-frequency harmonic code still packs 80 classes into only 8 embedding dimensions and therefore does not provide sufficient pairwise separation---a limitation captured by the small $\tau_{\mathrm{sep}} = 0.006$.

\begin{table}[t]
\centering
\caption{CIFAR-100 below-simplex stress test ($d{=}8$, $C{=}80$, 5 splits). Simplex ETF is infeasible; harmonic and learned prototypes remain defined.}
\label{tab:q2_below_simplex}
\small
\begin{tabular}{@{}llcc@{}}
\toprule
Representation & Best Scorer & AUROC & ACC \\
\midrule
CE          & ODIN        & $72.15 \pm 1.50$ & $60.97 \pm 1.13$ \\
Learned     & KNN\textsubscript{50}    & $70.43 \pm 0.85$ & $61.09 \pm 1.11$ \\
Harmonic    & MaxLogit    & $64.96 \pm 4.68$ & $35.39 \pm 7.42$ \\
\bottomrule
\end{tabular}
\end{table}

\paragraph{Summary for Q2.}
The $C$-gon construction successfully produces representations below the simplex barrier, confirming its theoretical purpose.
However, in the extreme regime of 80 classes at $d{=}8$, harmonic codes suffer high pairwise distance variability and reduced accuracy, consistent with the near-zero $\tau_{\mathrm{sep}}$ diagnostic from Corollary~\ref{cor:tau-sep}.
Learned prototypes (initialised from harmonic codes) close most of the gap to CE by adapting the centre locations during training.

%% =====================================================================
\subsection{Q3: Ratio Scores vs.\ Post-Hoc Scorers---Scope of the Theory}
\label{sec:q3}

The ratio scores $U$ and $U_2$ are central to the theory because they produce analytically tractable acceptance regions (Theorem~\ref{thm:geometry-U2}) and enable the FAR bounds of Theorem~\ref{thm:far-krr}.
However, the theory does not claim these are optimal OOD detectors in practice.
This subsection isolates the scoring question by comparing all 11 scorers on \emph{the same} frozen embeddings.

\paragraph{Dataset-level summary.}
Table~\ref{tab:main_results} reports the best scorer per representation across all datasets.
CE+ODIN has the highest AUROC on CIFAR+10 ($87.11$), CIFAR+50 ($86.09$), and CIFAR-100 ($72.15$).
CE with class-mean ratio scoring leads on PathMNIST ($85.27$).
BloodMNIST is the one dataset in this slice where prototype-trained embeddings narrowly lead: simplex+KNN\textsubscript{50} achieves $87.43 \pm 10.85$ vs.\ CE+MSP $87.05 \pm 9.00$, although the high split variance makes the difference statistically insignificant.

\begin{table}[t]
\centering
\caption{Best representation--scorer pairing per dataset ($d{=}8$). $\Delta$: gap between best prototype result and best CE result.}
\label{tab:main_results}
\scriptsize
\begin{tabularx}{\linewidth}{@{}lXXc@{}}
\toprule
Dataset & Best CE & Best Prototype & $\Delta$ \\
\midrule
CIFAR+10    & CE+ODIN: $87.11{\scriptstyle\pm1.22}$        & Learned+KNN\textsubscript{50}: $82.59{\scriptstyle\pm1.17}$  & $-4.52$ \\
CIFAR+50    & CE+ODIN: $86.09{\scriptstyle\pm0.80}$        & Simplex+MaxLogit: $82.51{\scriptstyle\pm0.79}$  & $-3.58$ \\
CIFAR-100   & CE+ODIN: $72.15{\scriptstyle\pm1.50}$        & Learned+KNN\textsubscript{50}: $70.43{\scriptstyle\pm0.85}$  & $-1.72$ \\
BloodMNIST  & CE+MSP: $87.05{\scriptstyle\pm9.00}$         & Simplex+KNN\textsubscript{50}: $87.43{\scriptstyle\pm10.85}$  & $+0.38$ \\
PathMNIST   & CE+$U$: $85.27{\scriptstyle\pm6.46}$ & Harmonic+KNN\textsubscript{50}: $82.44{\scriptstyle\pm4.49}$  & $-2.83$ \\
\bottomrule
\end{tabularx}
\end{table}

\paragraph{Ratio score failure analysis.}
Table~\ref{tab:ratio_gap} quantifies the gap between the ratio scores and the best post-hoc scorer on the \emph{same} learned-prototype embeddings.
On CIFAR+10, KNN\textsubscript{50} achieves 82.59 AUROC where $U$ achieves only 63.89---a gap of 18.7 points.
Similar gaps appear on CIFAR+50 ($-21.3$) and CIFAR-100 ($-18.7$).
The gap narrows on medical benchmarks: BloodMNIST ($-10.7$), and on PathMNIST learned-prototype ratio scores are essentially tied with the best scorer on the same embedding.

This discrepancy is attributable to the scoring rule rather than to the representation alone.
The same embeddings that yield 63.89 under $U$ achieve 82.59 under KNN\textsubscript{50}, indicating that the representation contains neighbourhood structure that the ratio score does not fully exploit.
This is consistent with the theory's scope: $U$ and $U_2$ are designed for analytic tractability (ball-shaped acceptance regions, closed-form FAR bounds), not for optimal empirical detection.

\begin{table}[t]
\centering
\caption{Ratio score gap on learned-prototype embeddings ($d{=}8$). The same representation can yield substantially different AUROC depending on the scorer.}
\label{tab:ratio_gap}
\scriptsize
\begin{tabularx}{\linewidth}{@{}lXcc@{}}
\toprule
Dataset & Best Scorer (AUROC) & Ratio $U$ (AUROC) & Gap \\
\midrule
CIFAR+10    & KNN\textsubscript{50}: $82.59$ & $63.89$ & $-18.7$ \\
CIFAR+50    & MSP: $81.70$                   & $60.39$ & $-21.3$ \\
CIFAR-100   & KNN\textsubscript{50}: $70.43$ & $51.77$ & $-18.7$ \\
BloodMNIST  & MSP: $82.30$                   & $71.57$ & $-10.7$ \\
PathMNIST   & $U$: $81.97$ & $81.97$ & $0.0$  \\
\bottomrule
\end{tabularx}
\end{table}

\paragraph{Why can CE+ODIN outperform prototype scorers?}
ODIN combines temperature scaling with input-gradient perturbation, which sharpens softmax confidence for in-distribution samples without retraining.
This score accesses the full Jacobian of the logit function---information that the distance-based ratio score discards.
The theory's Lipschitz bound (Theorem~2) guarantees that small input perturbations produce bounded score changes, but it does not predict whether gradient-based score refinement will improve detection.
The CE advantage is therefore outside the scope of the ratio-geometry theorems and should not be read as evidence against the geometric analysis.

\paragraph{When do prototype embeddings help post-hoc scorers?}
On CIFAR+10, prototype embeddings improve KNN\textsubscript{50} relative to CE embeddings: learned+KNN\textsubscript{50} achieves $82.59$ vs.\ CE+KNN\textsubscript{50} at $81.27$ ($+1.32$).
Similarly, on BloodMNIST, simplex+KNN\textsubscript{50} ($87.43$) exceeds CE+KNN\textsubscript{50} ($85.31$, $+2.12$).
On PathMNIST, harmonic+KNN\textsubscript{50} ($82.44$) also slightly exceeds CE+KNN\textsubscript{50} ($81.93$), although CE with class-mean ratio scoring has the highest AUROC overall ($85.27$).
These comparisons indicate that prototype training can improve neighbourhood structure for non-parametric scoring, even when the dedicated ratio scores underperform.

\paragraph{Summary for Q3.}
The ratio scores $U$ and $U_2$ are often not competitive with modern post-hoc scorers, losing 10--21 AUROC points on several natural-image and BloodMNIST embeddings.
PathMNIST is a notable exception: learned-prototype \(U\), \(U_2\), min-distance, and KNN\textsubscript{50} are all clustered near $82$ AUROC.
These findings do not conflict with the theory: the ratio scores serve their stated purpose of producing analytically tractable acceptance geometry with provable FAR bounds.
For practical open-set evaluation, we therefore report ratio scores alongside standard post-hoc scorers such as KNN or MSP on the same representation.

%% =====================================================================
\subsection{Connection to Theoretical Results}

We summarize how the main theoretical statements are reflected in the experiments.

\paragraph{Theorem~2 (Lipschitz bounds).}
The stability guarantee ensures bounded score variation under small input changes.
The low standard deviations across splits ($\pm 1$--$2$ AUROC for CIFAR variants) are consistent with Lipschitz-stable scoring, though we do not claim a direct causal test.

\paragraph{Theorem~8 (Ball geometry of $U_2$).}
The theorem predicts acceptance regions that are balls centred on prototype locations.
The near-zero barycentre norm and $\mathrm{CV}_r$ of prototype representations (Table~\ref{tab:geometry}) support A1$'$---zero barycentre and equal norm---which is the assumption needed for the ball characterisation.
The $\mathrm{CV}_d$ column is a separate simplex-defect diagnostic rather than a premise of Theorem~\ref{thm:geometry-U2}.

\paragraph{Theorem~\ref{thm:dichotomy} (Sharp dichotomy).}
The CIFAR-100 experiment tests the below-simplex barrier directly: at $C{=}80$ and $d{=}8$, exact regular-simplex one-distance geometry is impossible because $d<C{-}1$.
The elevated pairwise-distance variation and near-zero $\tau_{\mathrm{sep}}$ in Table~\ref{tab:geometry} are consistent with the theorem's message: the simplex barrier affects exact equidistance, not the existence of balanced equal-norm harmonic codes or compact $U_2$ acceptance geometry.

\paragraph{Section~5.2 (harmonic construction).}
The CIFAR-100 experiments (Table~\ref{tab:q2_below_simplex}) directly test the harmonic construction's viability below the simplex barrier, confirming that it produces usable representations at $d{=}8 \ll C{-}1{=}79$.

\paragraph{Corollary~\ref{cor:tau-sep} and Theorem~\ref{thm:far-krr} (separation and FAR bounds).}
The $\tau_{\mathrm{sep}}$ values in Table~\ref{tab:geometry} quantify the geometric tangent-ball separation margin.
Prototype methods achieve $\tau_{\mathrm{sep}} = 0.23$--$0.35$ on datasets where the dimension budget is sufficient (CIFAR+10, BloodMNIST), and $\tau_{\mathrm{sep}} \approx 0$ on CIFAR-100 where it is not, matching the intended diagnostic role of Corollary~\ref{cor:tau-sep}.

\paragraph{Corollary~\ref{cor:threshold} (Threshold selection).}
The val-calibrated protocol implements the corollary's recommendation: thresholds are chosen on held-out validation data, and no test-set information leaks into threshold selection.

%% =====================================================================
\subsection{Limitations}

The FAR bounds (Theorem~\ref{thm:far-krr}) apply to the analytic ratio geometry and do not extend to arbitrary post-hoc scores.
The ratio scores $U$ and $U_2$ are interpretable and theoretically informative but are often empirically outperformed by logit-based or nearest-neighbour scorers that exploit information the ratio scores discard.
Post-hoc scorer performance is dataset-, backbone-, dimension-, and calibration-dependent; we do not claim that prototype embeddings universally dominate CE embeddings under all scorers.
All paper-facing real-data experiments use ResNet-18 at $d{=}8$; larger backbones, higher embedding dimensions, and a full scorer-agnostic dimension sweep may shift the balance between representation methods.

%% ============================================================
%% 10. CONCLUSION
%% ============================================================
\section{Conclusion}\label{sec:conclusion}

This paper gives a theoretical framework for simplex-ratio open-set recognition that applies beyond the regular-simplex regime and covers arbitrary embedding dimensions. By introducing a hierarchy of prototype assumptions---from arbitrary distinct prototypes (A1) through balanced equal-norm codes (A1$'$) to the regular simplex (A1$''$)---we show that the core rejection geometry does not require the simplex assumption $d\ge C-1$.

For arbitrary distinct prototypes, we establish that $U$ is globally Lipschitz and that acceptance regions are compact sets bracketed by explicit prototype-centred ball unions. For balanced equal-norm codes (available in every $d\ge 2$), we prove that $U_2$-sublevel sets are exact unions of Euclidean balls, and a sharp dichotomy shows that perfect simplex symmetry can occur only when $d\ge C-1$. Below this threshold, the simplex-defect parameter $\lambda_j$ gives explicit local comparison bounds for the resulting degradation. We prove FAR upper bounds that decay exponentially with embedding dimension under natural isotropy assumptions, analyse the interplay between the three UCDSC loss components, and provide practical data-driven threshold procedures.

The experimental section tests three theory-facing questions rather than benchmark dominance: (Q1)~prototype-trained representations exhibit the core balanced-code diagnostics---near-zero barycentre norm and equal norms---and, when the dimension budget permits, smaller pairwise-distance variation than CE; (Q2)~harmonic constructions produce usable representations below the simplex barrier, with learned prototypes closing most of the gap to CE even at $d \ll C{-}1$; and (Q3)~the ratio scores $U$ and $U_2$ are analytically useful but scorer-dependent in practice, often trailing KNN/logit-based scorers on the same embeddings while remaining competitive on learned PathMNIST embeddings.

\textbf{Limitations and future work.}
Several assumptions remain limiting. Assumption A5 (sub-Gaussian score tail) is an idealisation that may not hold for poorly trained models or multi-modal class distributions; extending the analysis to weaker tail bounds is a natural next step. The RMS--AM ratio bound of $2/\sqrt{3}$ (Lemma~\ref{lem:rms-am}) holds under A1$''$ and may be improvable for specific values of $C$; for general balanced equal-norm codes we use the coarser global factor $\sqrt{2}$ together with the local comparison of Theorem~\ref{thm:local-comparison}. The harmonic construction provides a concrete low-dimensional family, but optimising the prototype configuration (e.g., via Grassmannian codes or group orbits) for specific $(C,d)$ pairs is an open direction. On the empirical side, post-hoc scorers remain essential, and closing the gap between theoretically motivated ratio scores and empirically effective detectors like ODIN is an important open problem. Tightening the connection between theoretical FAR bounds, learned representation geometry, and empirical rejection rates remains an open challenge that could guide future refinements of prototype-based OSR methods.
\appendix
\section{Full Proof of Lemma~\ref{lem:lipschitz}}\label{app:lipschitz}

Recall $d_j(z):=\|z-s_j\|$, $d_{\min}(z):=\min_j d_j(z)$, and
$\mu(z):=\frac{1}{C-1}\sum_{k\neq j^*(z)} d_k(z)$, where $j^*(z)\in\arg\min_j d_j(z)$.

\textbf{(i) Pointwise bound $d_k(z)\ge D/2$ for $k\neq j^*(z)$, and hence $\mu(z)\ge D/2$.}
Fix $z$ and let $j^*=j^*(z)$. Under A1$''$ (simplex), the common edge length is
$D=\|s_{j^*}-s_k\|$ for all $k\neq j^*$. By the triangle inequality,
\[
D=\|s_{j^*}-s_k\|\le \|z-s_{j^*}\|+\|z-s_k\|=d_{\min}(z)+d_k(z).
\]
Since $d_k(z)\ge d_{\min}(z)$, we have $D\le 2d_k(z)$, hence $d_k(z)\ge D/2$ for all $k\neq j^*$.
Averaging over $k\neq j^*$ yields $\mu(z)\ge D/2$.

\textbf{(ii) Global Lipschitzness of $U(z)=d_{\min}(z)/\mu(z)$.}
First, $d_{\min}$ is $1$-Lipschitz because it is the pointwise minimum of $1$-Lipschitz maps:
\[
|d_{\min}(z)-d_{\min}(z')|\le \|z-z'\|.
\]
Next, define for each $j\in[C]$ the function
\[
\mu_j(z):=\frac{1}{C-1}\sum_{k\neq j} d_k(z).
\]
Each $\mu_j$ is $1$-Lipschitz as an average of $C-1$ functions each $1$-Lipschitz, with coefficients $1/(C-1)$ summing to $1$.
Moreover, because removing the smallest distance maximizes the remaining average,
\[
\mu(z)=\mu_{j^*(z)}(z)=\max_{j\in[C]} \mu_j(z).
\]
The pointwise maximum of $1$-Lipschitz functions is $1$-Lipschitz, so
$|\mu(z)-\mu(z')|\le \|z-z'\|$.

Finally, note $d_{\min}(z)\le \mu(z)$ and $d_{\min}(z')\le \mu(z')$ (each non-nearest distance is $\ge d_{\min}$).
Using the identity (for $0\le a\le b$, $0\le a'\le b'$):
\[
\begin{aligned}
\left|\frac{a}{b}-\frac{a'}{b'}\right|
&=\frac{|ab'-a'b|}{bb'}\\
&\le \frac{|a-a'|b'+a'|b-b'|}{bb'}\\
&\le \frac{|a-a'|+|b-b'|}{\min\{b,b'\}},
\end{aligned}
\]
where the last step uses $a'\le b'$.
Applying (i) so that $\min\{\mu(z),\mu(z')\}\ge D/2$, we obtain
\[
\begin{aligned}
|U(z)-U(z')|
&\le \frac{|d_{\min}(z)-d_{\min}(z')|+|\mu(z)-\mu(z')|}{D/2}\\
&\le \frac{2}{D}\Bigl(\|z-z'\|+\|z-z'\|\Bigr)\\
&= \frac{4}{D}\|z-z'\|.
\end{aligned}
\]
Thus $U$ is globally Lipschitz with constant $L_U=4/D$.

\section{Full Proof of Theorem~\ref{thm:geometry-U2}}\label{app:geometry}

\noindent\textbf{Full derivation of Lemma~\ref{lem:rms-am} ($U$--$U_2$ equivalence).}
Fix $z\in\R^d$ and write $j^*=j^*(z)$. Let $a_k:=d_k(z)$ for $k\neq j^*$ denote the $C-1$ non-nearest distances.
Define
\[
\mathrm{AM}:=\frac{1}{C-1}\sum_{k\neq j^*} a_k,\qquad
\mathrm{RMS}:=\sqrt{\frac{1}{C-1}\sum_{k\neq j^*} a_k^2}.
\]
If $d_{\min}(z)=0$, then $U(z)=\sqrt{U_2(z)}=0$ and both bounds are immediate.
Assume below that $d_{\min}(z)>0$.
By construction, $U(z)=d_{\min}(z)/\mathrm{AM}$ and $\sqrt{U_2(z)}=d_{\min}(z)/\mathrm{RMS}$, hence
$U(z)/\sqrt{U_2(z)}=\mathrm{RMS}/\mathrm{AM}$.

\emph{Lower bound.}
Always $\mathrm{AM}\le \mathrm{RMS}$ (by Cauchy--Schwarz), so $\sqrt{U_2(z)}\le U(z)$.

\emph{Upper bound via bounded spread.}
Let $m:=\min_{k\neq j^*} a_k$ and $M:=\max_{k\neq j^*} a_k$.
We claim $M/m\le 3$. Indeed, for any $k\neq j^*$:
\[
a_k=d_k(z)\le d_{j^*}(z)+\|s_k-s_{j^*}\|=d_{\min}(z)+D\le m + D,
\]
since $m\ge d_{\min}(z)$. Also, by Lemma~\ref{lem:lipschitz}(i), $m\ge D/2$, hence $D\le 2m$ and
$a_k\le m+2m=3m$. Taking the maximum over $k$ gives $M\le 3m$.

Now scale by $m$ so that $x_k:=a_k/m\in[1,3]$.
Then
\[
\frac{\mathrm{RMS}}{\mathrm{AM}}
=
\frac{\sqrt{\frac{1}{C-1}\sum x_k^2}}{\frac{1}{C-1}\sum x_k}.
\]
For a fixed mean, the second moment over $[1,3]$ is maximized when all mass lies at the interval endpoints.
Hence the worst case for $\sqrt{\E[x^2]}/\E[x]$ is attained by a two-point distribution on $\{1,3\}$.
Writing $p$ for the fraction of entries equal to $3$, this yields
\[
\begin{aligned}
\frac{\mathrm{RMS}}{\mathrm{AM}}
&\le
\max_{p\in[0,1]}\frac{\sqrt{(1-p)\cdot 1^2+p\cdot 3^2}}{(1-p)\cdot 1+p\cdot 3}\\
&=
\max_{p\in[0,1]}\frac{\sqrt{1+8p}}{1+2p}
=\frac{2}{\sqrt{3}},
\end{aligned}
\]
achieved at $p=1/4$.
Therefore $U(z)\le \frac{2}{\sqrt{3}}\sqrt{U_2(z)}$.
The set inclusions now follow in two steps:
\[
U_2\le \frac{3\theta^2}{4}
\;\Rightarrow\;
U\le\frac{2}{\sqrt{3}}\sqrt{U_2}
\le\frac{2}{\sqrt{3}}\cdot\frac{\theta\sqrt{3}}{2}
=\theta,
\]
and
\[
U\le\theta
\;\Rightarrow\;
\sqrt{U_2}\le U\le\theta
\;\Rightarrow\;
U_2\le\theta^2.
\]
\qed

\medskip

\textbf{Part (i).}
$\|z-s_j\|^2=\|z\|^2+R^2-2\langle z,s_j\rangle$.
Hence $d_{\min}^{(2)}=\|z\|^2+R^2-2\alpha(z)$.
Since $\sum_j s_j=0$: $\sum_j\|z-s_j\|^2=C\|z\|^2+CR^2$.
Therefore:
\begin{align*}
\nu(z)&=\frac{1}{C-1}(C\|z\|^2+CR^2-\|z\|^2-R^2+2\alpha(z))\\
&=\|z\|^2+R^2+\frac{2}{C-1}\alpha(z).
\end{align*}
Note $\alpha(z)\ge 0$ (since the $C$ values $\langle z,s_j\rangle$ sum to zero), so $\nu(z)>0$.

\textbf{Part (ii).}
If $\rho=0$, then $U_2(z)=0$ iff the nearest squared distance is zero, so the sublevel set is exactly $\{s_1,\dots,s_C\}$, matching $\mathbb B_j(0)=\{s_j\}$.
Now assume $\rho\in(0,1)$.
$U_2(z)\le\rho$ rearranges to $\alpha(z)\ge\frac{t(\rho)}{2}(\|z\|^2+R^2)$, i.e.\
$\exists j:\langle z,s_j\rangle\ge\frac{t}{2}(\|z\|^2+R^2)$.
Rearranging: $\|z\|^2-\frac{2}{t}\langle z,s_j\rangle+R^2\le 0$.
Completing the square:
$\|z-s_j/t\|^2=\|z\|^2-\frac{2}{t}\langle z,s_j\rangle+R^2/t^2$,
so the condition becomes $\|z-s_j/t\|^2\le R^2(1/t^2-1)=:r_j^2$.
This is exactly $z\in\mathbb B_j(\rho)$.
Since $t\in(0,1)$, each ball has finite radius; the union is compact.

\textbf{Part (iii).}
On $\|z\|=R$: $\frac{t}{2}(R^2+R^2)=tR^2$.

\textbf{Balanced-code comparison used in Part (iv).}
Assume A1$'$ and keep the notation
\[
\mathrm{AM}:=\frac{1}{C-1}\sum_{k\neq j^*} a_k,\qquad
\mathrm{RMS}:=\sqrt{\frac{1}{C-1}\sum_{k\neq j^*} a_k^2}.
\]
Let $r:=\|z\|$. Since $\|s_k\|=R$ for every $k$, triangle inequality gives $a_k\le r+R$, hence
\[
\mathrm{RMS}^2=\frac{1}{C-1}\sum_{k\neq j^*} a_k^2
\le (r+R)\frac{1}{C-1}\sum_{k\neq j^*} a_k
=(r+R)\mathrm{AM}.
\]
By Part (i),
\[
\mathrm{RMS}^2=\nu(z)=r^2+R^2+\frac{2}{C-1}\alpha(z)\ge r^2+R^2,
\]
because $\alpha(z)\ge 0$. Therefore
\[
\frac{U(z)}{\sqrt{U_2(z)}}=\frac{\mathrm{RMS}}{\mathrm{AM}}
\le \frac{r+R}{\mathrm{RMS}}
\le \frac{r+R}{\sqrt{r^2+R^2}}
\le \sqrt{2}.
\]
Thus $U(z)\le \sqrt{2}\sqrt{U_2(z)}$ under A1$'$.

\textbf{Part (iv).}
By Proposition~\ref{prop:global-comparison}, $\{U\le\theta\}\subseteq\{U_2\le\theta^2\}$.
For the inner inclusion under A1$'$,
\[
U_2\le \theta^2/2
\;\Rightarrow\;
U\le \sqrt{2}\sqrt{U_2}
\le \sqrt{2}\cdot \frac{\theta}{\sqrt{2}}
=\theta.
\]
Hence $\{U_2\le\theta^2/2\}\subseteq\{U\le\theta\}$.
Under the additional simplex assumption A1$''$, Lemma~\ref{lem:rms-am} improves this to $\{U_2\le 3\theta^2/4\}\subseteq\{U\le\theta\}$.
Apply part (ii) to the relevant sublevel sets.

\section{Full Proof of Theorem~\ref{thm:far-krr}}\label{app:far}

\textbf{Part (i): KRR control.}
Let $m_y:=\E[U(Z)\mid Y=y]$ and $t_{\varepsilon,y}:=\tau_y\sqrt{2\log(1/\varepsilon)}$.
By Lemma~\ref{lem:quantile}, any threshold $\bar\theta_\varepsilon\ge\max_y(m_y+t_{\varepsilon,y})$ satisfies
$\Prob(U(Z)>\bar\theta_\varepsilon\mid Y=y)\le\varepsilon$ for each $y$.
If $\bar\theta_\varepsilon\le 1$, take $\theta_\varepsilon:=\bar\theta_\varepsilon$ and average:
\[
\KRR(\theta_\varepsilon)\le\sum_y\Prob(Y=y)\cdot\varepsilon=\varepsilon.
\]
If $\bar\theta_\varepsilon>1$, take $\theta_\varepsilon:=1$; since $U(Z)\in[0,1]$, this gives $\KRR(\theta_\varepsilon)=0\le\varepsilon$.
The common-scale version follows from the conservative bound $t_{\varepsilon,y}\le\tau\sqrt{2\log(1/\varepsilon)}=L_U\sigma\sqrt{2\log(1/\varepsilon)}$.

\textbf{Part (ii): FAR via geometry.}
$\sqrt{U_2}\le U$ (Proposition~\ref{prop:global-comparison}) $\Rightarrow$ $\{U\le\theta\}\subseteq\{U_2\le\theta^2\}$.
Apply Theorem~\ref{thm:geometry-U2}(ii).

\textbf{Part (iii): FAR under (U-Norm) + A6$'$ for $\theta\in[0,1)$.}
\emph{Step 1: Reduce to spherical caps.}
By the convention $\Phi(0):=0$ in Assumption~\ref{ass:unknown}, and since $\theta<1$, the point $0$ is rejected because $U(0)=1$.
Thus the acceptance event $\{U\le\theta\}$ forces $Z_U\neq 0$, where $\|\Phi(Z_U)\|=R$.
By Theorem~\ref{thm:geometry-U2}(iii), on the sphere $\|z\|=R$:
\[
U_2(z)\le\theta^2\iff\exists\, j\in[C]:\langle z,s_j\rangle\ge t(\theta^2)R^2.
\]
Let $\widetilde{Z}=\Phi(Z_U)/R\in S^{d-1}$ and $v_j=s_j/R$.
The condition becomes $\exists\, j:\langle\widetilde{Z},v_j\rangle\ge t(\theta^2)$.

\emph{Step 2: Union bound.}
Since acceptance forces $Z_U\neq 0$ (Step 1), $\FAR(\theta)\le\Prob(U_2(\Phi(Z_U))\le\theta^2)$ by Proposition~\ref{prop:global-comparison}.
On $\|z\|=R$, this is $\Prob(\bigcup_{j=1}^C\{\langle\widetilde{Z},v_j\rangle\ge t\}\mid Z_U\neq 0)\Prob(Z_U\neq 0)$.
Since $\Prob(Z_U\neq 0)\le 1$:
\begin{align*}
\FAR(\theta)
&\le\Prob\!\left(\bigcup_{j=1}^C\{\langle\widetilde{Z},v_j\rangle\ge t\}\;\middle|\;Z_U\neq 0\right)\\
&\le\sum_{j=1}^C\Prob(\langle\widetilde{Z},v_j\rangle\ge t\mid Z_U\neq 0). \tag{union bound}
\end{align*}

\emph{Step 3: Apply (U-Norm) and A6$'$.}
By (U-Norm), for each $j$:
\begin{align*}
\Prob(\langle\widetilde{Z},v_j\rangle\ge t\mid Z_U\neq 0)
&\le \pcap(t;d)+\delta_{\mathrm{cap}}.
\end{align*}
This is precisely where $\delta_{\mathrm{cap}}$ enters: it measures how far the true angular distribution of normalized unknowns deviates from the ideal uniform sphere.
By A6$'$, $\pcap(t;d)\le\exp(-(d-1)t^2/2)$.
Summing over $j$:
\[
\FAR(\theta)\le C\bigl(\exp(-(d-1)t(\theta^2)^2/2)+\delta_{\mathrm{cap}}\bigr).
\]

\emph{Step 4: Lower bounds under exact isotropy.}
Assume $\delta_{\mathrm{cap}}=0$ and $\Prob(Z_U\neq 0)=1$.
By Theorem~\ref{thm:geometry-U2}(iv), under A1$'$ we have $\{U_2\le\theta^2/2\}\subseteq\{U\le\theta\}$.
Hence
\[
\Prob(U_2(\Phi(Z_U))\le\theta^2/2)\le \FAR(\theta).
\]
On the sphere $\|z\|=R$, the event $\{U_2(z)\le\theta^2/2\}$ is the union of caps
\[
\bigcup_{j=1}^C\{\langle \widetilde Z,v_j\rangle\ge t(\theta^2/2)\}.
\]
The probability of this union is at least the probability of any one cap, so exact isotropy gives
\[
\pcap(t(\theta^2/2);d)\le \FAR(\theta).
\]

If in addition A1$''$ holds, then Lemma~\ref{lem:rms-am} implies the smaller inner inclusion $\{U_2\le(\theta/3)^2\}\subseteq\{U\le\theta\}$.
The corresponding cap threshold satisfies $t((\theta/3)^2)\ge t(1/9)=8(C-1)/(9C-8)\ge 4/5>1/\sqrt{2}$.
Thus each cap half-angle is strictly less than $\pi/4$.
Simplex cap centers satisfy $\langle v_i,v_j\rangle=-1/(C-1)\le 0$, so their pairwise angular separation is at least $\pi/2$.
Hence these simplex caps are pairwise disjoint, and exact isotropy yields the disjoint-cap simplex lower bound
\[
C\,\pcap(t((\theta/3)^2);d)\le \FAR(\theta).
\]

\textbf{Part (iv): FAR under (U-RadIso) + A6$'$ for $\theta\in[0,1)$.}
Under (U-RadIso), $\Phi=\mathrm{id}$ (Assumption~\ref{ass:unknown}), so $\FAR(\theta)=\Prob(U(Z_U)\le\theta)$.
By Proposition~\ref{prop:global-comparison}, $\{U\le\theta\}\subseteq\{U_2\le\theta^2\}$, so $\FAR(\theta)\le\Prob(U_2(Z_U)\le\theta^2)$.
Write $Z_U=R_UV_U$ with $R_U=\|Z_U\|$ and $V_U\mid R_U$ uniform on $S^{d-1}$.
For $R_U=r>0$:
$\langle Z_U,s_j\rangle\ge\frac{t(\theta^2)}{2}(r^2+R^2)\iff\langle V_U,v_j\rangle\ge\tau_\theta(r)$
where $\tau_\theta(r)=\frac{t(\theta^2)}{2}\cdot\frac{r^2+R^2}{rR}$.
For $R_U=0$: $Z_U=0$, and since $\theta<1$ we have $U(0)=1>\theta$, so the point is rejected.
Condition on $R_U=r$, apply the union bound over $j$, use isotropy of $V_U\mid R_U$, and apply A6$'$:
\[
\FAR(\theta)\le C\,\E\!\left[\exp\!\left(-\frac{d-1}{2}\min\{1,\tau_\theta(R_U)\}^2\right)\right].
\]
The clipping $\min\{1,\cdot\}$ ensures $\pcap$ receives an argument in $[0,1]$; it is necessary when $R_U$ is far from $R$, making $\tau_\theta(R_U)>1$.

\section{Proof of Theorem~\ref{thm:combined}}\label{app:combined}

Recall that Theorem~\ref{thm:combined} is stated under A1--A2 and uses
\[
\Delta:=\min_{i\neq j}\|s_i-s_j\|,\qquad D_{\max}:=\max_{i\neq j}\|s_i-s_j\|.
\]

\textbf{Part (i).}
$d_{\min}(f_i)\le\|f_i-s_{y_i}\|$ (the nearest prototype is at most as far as the true one).
By the proof of Theorem~\ref{thm:generic-ball}, $\mu(f_i)\ge \Delta/2$.
Hence
\[
U(f_i)\le \frac{2\|f_i-s_{y_i}\|}{\Delta}.
\]

\textbf{Part (ii).}
$L_o=0$ implies $\|f_k^{\mathrm{bg}}-s_j\|\ge\sqrt{m}$ for all $j,k$ (Proposition~\ref{prop:outlier}).
Hence $d_{\min}(f_k^{\mathrm{bg}})\ge\sqrt{m}$.
The triangle inequality gives $\mu(f_k^{\mathrm{bg}})\le d_{\min}(f_k^{\mathrm{bg}})+D_{\max}$, and therefore
\[
U(f_k^{\mathrm{bg}})
\ge
\frac{\sqrt{m}}{D_{\max}+\sqrt{m}}.
\]
If $\sqrt{m}\ge D_{\max}$, then $\sqrt{m}/(D_{\max}+\sqrt{m})\ge D_{\max}/(2D_{\max})=1/2$.

\textbf{Part (iii).}
Part (i) gives the known-feature upper bound $2\max_i\|f_i-s_{y_i}\|/\Delta$.
Part (ii) gives the background lower bound $\sqrt{m}/(D_{\max}+\sqrt{m})$.
Any threshold strictly between these two quantities separates the two sets, and the condition that such a threshold exists is exactly~\eqref{eq:sep-condition}.
In the simplex specialization A1$''$, one has $\Delta=D_{\max}=D$, recovering the simpler formulas stated in the main text.

%% ============================================================
%% REFERENCES (using thebibliography for self-containment)
%% ============================================================

\end{document}

%% file: table1_sigma_proxy.tex
\begin{tabular}{ccccc}
\toprule
Class & $n$ & $\hat m_y$ & $\hat\tau_y$ & $\hat\sigma_y$ \\
\midrule
0 & 30 & 0.038 & 0.0101 & 0.207 \\
1 & 30 & 0.040 & 0.0112 & 0.228 \\
2 & 30 & 0.042 & 0.0165 & 0.338 \\
3 & 30 & 0.039 & 0.0098 & 0.200 \\
\midrule
\multicolumn{3}{l}{$\hat\sigma=\max_y\hat\sigma_y$} & \multicolumn{2}{c}{$ 0.338 $} \\
\multicolumn{3}{l}{$\hat\theta_{0.01}$} & \multicolumn{2}{c}{$ 0.092 $} \\
\bottomrule
\end{tabular}

%% file: table2_far_vs_d.tex
\begin{tabular}{rccc}
\toprule
$d$ & FAR (empirical) & FAR (bound) & $\log_{10}$(bound) \\
\midrule
3 & 0.1048 & 1 & $0.0$ \\
5 & 0.0088 & 0.6645 & $-0.2$ \\
8 & $<\!2.0\!\times\!10^{-4}$ & 0.1729 & $-0.8$ \\
16 & $<\!2.0\!\times\!10^{-4}$ & 0.0048 & $-2.3$ \\
32 & $<\!2.0\!\times\!10^{-4}$ & $3.6\!\times\!10^{-6}$ & $-5.4$ \\
64 & $<\!2.0\!\times\!10^{-4}$ & $2.1\!\times\!10^{-12}$ & $-11.7$ \\
\bottomrule
\end{tabular}